\documentclass{article}
\usepackage{makecell}
\usepackage{enumitem}
\usepackage{microtype}
\usepackage{graphicx}
\usepackage{subcaption}
\usepackage{booktabs} 
\usepackage{newunicodechar}
\usepackage{amssymb} 
\usepackage{bbm}  
\usepackage{amsmath} 

\usepackage{graphicx}
\usepackage{subcaption} 
\newunicodechar{↵}{\ensuremath{\hookleftarrow}}
\usepackage[T1]{fontenc}
\usepackage{hyperref}

\usepackage[accepted]{icml2026}

\usepackage{amssymb}   
\usepackage{multirow}  
\usepackage{graphicx}  
\usepackage{amsmath}
\usepackage{amssymb}
\usepackage{mathtools}
\usepackage{amsthm}

\usepackage[capitalize,noabbrev]{cleveref}

\theoremstyle{plain}
\newtheorem{theorem}{Theorem}[section]

\newtheorem{lemma}[theorem]{Lemma}

\theoremstyle{definition}

\theoremstyle{remark}

\usepackage[textsize=tiny]{todonotes}

\icmltitlerunning{Cross-sample Consistency Regularization for Sparse Autoencoders}

\definecolor{malicious}{RGB}{68,1,84}
\definecolor{safe}{RGB}{253,231,37}

\begin{document}

\twocolumn[
  \icmltitle{C$^{2}$R: Cross-sample Consistency Regularization Mitigates Feature Splitting and Absorption in Sparse Autoencoders}



  \icmlsetsymbol{equal}{*}

  \begin{icmlauthorlist}
    \icmlauthor{Haoran Jin}{ustc,cogskl}
    \icmlauthor{Xiting Wang}{gsai}
    \icmlauthor{Shijie Ren}{gsai}
    \icmlauthor{Hong Xie}{ustc,cogskl}
    \icmlauthor{Defu Lian}{ustc,cogskl}
  \end{icmlauthorlist}

  \icmlaffiliation{ustc}{University of Science and Technology of China}
  \icmlaffiliation{cogskl}{State Key Laboratory of Cognitive Intelligence}
  \icmlaffiliation{gsai}{Gaoling School of Artificial Intelligence, Renmin University of China}

  \icmlcorrespondingauthor{Xiting Wang}{xitingwang@ruc.edu.cn}
  \icmlcorrespondingauthor{Hong Xie}{hongx87@ustc.edu.cn}
  \icmlcorrespondingauthor{Defu Lian}{liandefu@ustc.edu.cn}

  \icmlkeywords{Machine Learning, ICML}

  \vskip 0.3in
]



\printAffiliationsAndNotice{}  

\begin{abstract}
Sparse Autoencoders (SAEs) are widely used to interpret large language models by decomposing activations into sparse, human-understandable features, but scaling to large dictionaries exposes fundamental challenges. Systematic studies reveal pervasive feature splitting that fragments coherent concepts into non-atomic latents and widespread feature absorption that creates arbitrary exceptions in general features, severely compromising latent reliability. These issues stem from inconsistent latent assignment across samples: without cross-sample constraints, per-sample optimization often allows a single underlying concept to be inconsistently distributed across multiple redundant or interfering latents. To address this, we introduce C$^2$R (\underline{\textbf{C}}ross-sample \underline{\textbf{C}}onsistency \underline{\textbf{R}}egularization). C$^2$R explicitly encourages that each semantic feature is consistently represented by a unified latent across the batch by penalizing the co-activation of directionally similar latents. Comprehensive evaluation demonstrates that C$^2$R effectively mitigates both splitting and absorption while, crucially, preserving reconstruction fidelity, providing a principled solution that enhances latent interpretability without degrading model performance. Source code is available\footnote{\url{https://github.com/hr-jin/Cross-sample-Consistency-Regularization}}.
\end{abstract}

\section{Introduction}


\begin{table}[htbp]
\centering
\resizebox{\columnwidth}{!}{%
\begin{tabular}{c|c|c|c|c|c}
\hline
\multirow{2}{*}{SAE Constraints} & \multicolumn{2}{c|}{Theoretical Guarantee} & \multicolumn{2}{c|}{Intuitively Solved} & \multirow{2}{*}{\makecell{Reconstruction\\Preservation}} \\ \cline{2-5}
 & Splitting & Absorption & Splitting & Absorption &  \\ \hline
$\ell_1$ & $\times$ & $\times$ & $\times$ & $\times$ & $\checkmark$  \\
TopK & $\times$ & $\times$ & $\times$ & $\times$ & $\checkmark$  \\
Batch TopK & $\times$ & $\times$ & $\times$ & $\times$ & $\checkmark$ \\
Matryoshka & $\times$ & $\times$ & $\checkmark$ & $\checkmark$ & $\times$  \\
Ort & $\times$ & $\times$  & $\times$ & $\checkmark$ & $\checkmark$  \\
Ours & $\checkmark^*$ & $\checkmark^*$ & $\checkmark$ & $\checkmark$ & $\checkmark$ \\ \hline
\end{tabular}%
}
\caption{Comparison of different SAE constraints in terms of theoretical guarantees, intuitive solutions, and reconstruction fidelity. Our cross-sample consistency regularization uniquely offers a theoretical guarantee against splitting and absorption and preserves reconstruction fidelity. $^*$Theoretical guarantee holds under the condition in Eq.~\ref{eq:feature_disp_condition}, which is empirically satisfied in 88.1\% of absorption pairs (see Appendix~\ref{app:eq13_verification}).}
\label{tab:comparison}
\vspace{-10pt}
\end{table}

Sparse Autoencoders (SAEs) have emerged as a powerful tool for interpreting large language models (LLMs), breaking down complex internal representations into sparse, interpretable features~\citep{huben2023sparse, bricken2023monosemanticity}. These features provide valuable insights into reasoning, alignment~\citep{zhao-etal-2025-steering, yeo2025understanding, wang2025resa}, knowledge awareness, hallucinations~\citep{ferrandoknow}, and cross-model feature spaces~\citep{lan2025quantifyingfeaturespaceuniversality} of LLMs. This approach is grounded in the hypothesis that the semantic \emph{features} in a model’s representation are more effectively captured by an overcomplete sparse basis~\citep{olshausen1997sparse} than by dense neuron activations, which tend to be polysemantic. Ideally, each latent in an SAE corresponds to a single, human-interpretable concept.

While SAEs are effective, they face significant challenges, specifically feature splitting~\citep{bricken2023monosemanticity, leasksparse} and feature absorption~\citep{chanin2024absorption}, which undermine the reliability of learned latents.
Feature splitting fragments coherent, high-level concepts into overly specific pieces. For example, a single ``Mathematics'' feature might break down into separate latents for ``Algebra,'' ``Geometry,'' and others~\citep{chanin2024absorption}. Although these granular features are interpretable, this fragmentation is problematic because it obscures the true high-level concept the model functionally uses. Moreover, this is inefficient: the dictionary wastes capacity on redundant variations of known concepts instead of finding new ones.
Feature absorption, on the other hand, creates ``holes'' in general features when specific latents capture their activations. A ``starts with S'' latent, for instance, might fail to activate on ``short'' or ``small'' because token-specific latents absorb the signal. This effectively changes the latent to ``starts with S (except for short/small),'' distorting the intended pattern. 
\citet{leasksparse} show that splitting scales with model size, while \citet{chanin2024absorption} find that absorption affects hundreds of LLM SAEs. These systematic failures undermine the utility of SAEs for critical tasks like causal analysis and circuit discovery.


We argue that these failures arise from a mismatch between the hierarchical nature of language model features and the local scope of standard sparsity constraints.
Real-world concepts are inherently hierarchical~\citep{bussmannlearning}, yet common objectives like $\ell_1$~\citep{bricken2023monosemanticity} or TopK~\citep{gaoscaling} enforce sparsity on a per-sample basis, which potentially penalizes the hierarchical structure.
activating both uses more of the sparsity budget than activating the child feature alone. Consequently, the optimizer suppresses the parent latent and forces the child latent to take over its role to keep the active count low. Similarly, regarding feature splitting, the objective does not distinguish between activating a general latent or a specific one. The SAE allows disjoint latents to handle different contexts of a single concept, as nothing ensures consistent latent assignment across samples. 
Solving these issues, therefore, requires looking beyond per-sample optimization to enforce cross-sample consistency in how latents are selected.

To address this issue, we propose C$^2$R (\underline{\textbf{C}}ross-sample \underline{\textbf{C}}onsistency \underline{\textbf{R}}egularization).
This objective builds on the geometry of the Minkowski inequality~\citep{Gruber1979} and the strict convexity of the $\ell_2$ norm.
It exploits the fact that the sum of the norms of separate vectors strictly exceeds the norm of their sum: $\|u\|_2 + \|v\|_2 > \|u+v\|_2$ for non-aligned vectors.
By applying this constraint across the batch dimension, C$^2$R makes it expensive to split a concept into multiple disjoint latents.
This formulation penalizes spreading semantic information across redundant latents, driving the SAEs to consolidate activations into a single, consistent latent without supervision.

Our contributions are threefold:
\begin{itemize}[topsep=0pt, itemsep=8pt, parsep=0pt]
    \item \textbf{Theoretical diagnosis:} We identify the lack of cross-sample consistency in per-sample sparsity objectives as the root cause of feature splitting and absorption, providing a unified formal analysis of these phenomena. \looseness=-1
    \item \textbf{Principled objective:} We propose C$^2$R, a novel regularization objective that utilizes decoder geometry and batch-level statistics to enforce consistent latent selection, effectively distinguishing between true polysemanticity and harmful redundancy.
    \item \textbf{Empirical validation:} We demonstrate that C$^2$R significantly mitigates splitting and absorption, achieving better feature hierarchy without compromising reconstruction fidelity compared to state-of-the-art baselines.
\end{itemize}

\section{Related Work}
\label{sec:background}

\subsection{Sparse Autoencoders}

Sparse Autoencoders (SAEs) are grounded in the linear representation hypothesis, which posits that the dense activation space of a language model is constructed from the superposition of sparse, discernible concepts, referred to as \textit{features}. The goal of an SAE is to recover these ground-truth features by learning a dictionary of \textit{latents}. Ideally, there exists a one-to-one mapping where each learned latent corresponds precisely to a single meaningful feature.

Formally, given an input activation vector $x \in \mathbb{R}^{d_{model}}$ (e.g., from a Transformer's residual stream), an SAE projects $x$ into a higher-dimensional sparse latent code $f \in \mathbb{R}^{d_{dict}}$, where $d_{dict} \gg d_{model}$. The encoding process is parameterized by an encoder weight matrix $W_e \in \mathbb{R}^{d_{dict} \times d_{model}}$ and a bias $b_e$:
\begin{equation}
    f = \phi(W_e x + b_e),
\end{equation}
where $\phi$ is a non-linear activation function, typically ReLU, TopK, or JumpReLU, designed to induce sparsity. The input is then reconstructed via a linear decoder $W_d \in \mathbb{R}^{d_{model} \times d_{dict}}$:
\begin{equation}
    \hat{x} = W_d f + b_d.
\end{equation}
The training objective minimizes a combination of reconstruction error and a sparsity penalty:
\begin{equation}
    \mathcal{L}_{\text{SAE}}(x) = \|x - \hat{x}\|_2^2 + \lambda \mathcal{R}(f).
\end{equation}
Common choices for the regularizer $\mathcal{R}(f)$ include the $\ell_1$ norm \citep{bricken2023monosemanticity} or the auxiliary loss associated with TopK constraints \citep{gaoscaling}. 
Various architectural improvements have been proposed to enhance SAE quality. Gated SAEs \citep{rajamanoharan2024improving} and JumpReLU SAEs \citep{rajamanoharan2024jumping} introduce learnable thresholds to improve the fidelity-sparsity frontier. However, these methods focus on the per-sample activation sparsity rather than enforcing the hierarchical structure of the SAE. \looseness=-1

\subsection{Structural SAEs}
More recently, approaches attempting to structure the latent space have emerged. Batch TopK SAEs \citep{leasksparse} relax the rigid per-sample TopK constraint to a batch-level aggregate, allowing for variable sparsity across samples. While this improves reconstruction, it lacks any mechanism to enforce hierarchical structure among the latents and still faces feature absorption and splitting challenges. 

Matryoshka SAEs~\citep{bussmannlearning} enforce a nested structure where subsets of latents are trained to approximate the input at different sparsity levels. While this creates a hierarchy, it compromises reconstruction fidelity and lacks a clear theoretical explanation for why it would fix the optimization issues of standard sparsity objectives. OrtSAE~\citep{korznikov2025ort} addresses feature splitting and composition by penalizing the cosine similarity between decoder weights. However, this approach tries to handle absorption indirectly via the decoder geometry, rather than addressing the encoder activation patterns where absorption is formally defined. In contrast, our method targets the latent activations directly to penalize redundancy, offering a theoretically guaranteed solution derived from the formal definitions of splitting and absorption.

\subsection{Minkowski Inequality~\citep{Gruber1979}}
For any real number $p \geq 1$ and any two real sequences $a = (a_{1}, a_{2}, \ldots, a_{n})$ and $b = (b_{1}, b_{2}, \ldots, b_{n})$, their $p$-norms satisfy
\vspace{-5pt}
\begin{equation}
\label{eq:minkowski}
\|a + b\|_{p} \leq \|a\|_{p} + \|b\|_{p}.
\end{equation}
When the activations of a single semantic feature are distributed across multiple redundant SAE latents, the combined $p$-norm of their activations exceeds that of a single latent capturing the same feature. This inequality motivates our regularization term that penalizes redundant feature allocation across latents, thereby constraining feature splitting and absorption.
\section{Unified Problem Formulation}
\label{sec:formulation}

In this section, we propose a geometric framework that unifies feature splitting and absorption. Rather than viewing them as separate pathologies, we model them as instances of latent redundancy arising from perturbed basis directions. We introduce a \textbf{redundancy parameter $\alpha$} to quantify the extent to which a semantic feature ``leaks'' into varying latents, allowing us to derive a single consistency condition that prevents both failure modes.

Let $L_1$ denote the ideal latent direction that captures the complete semantic feature $F$, and let $L_2$ denote an orthogonal direction that captures residual, non-feature components.  
Both $L_1$ and $L_2$ are unit vectors and orthogonal to each other:
\begin{equation}
\|L_1\| = \|L_2\| = 1, \quad L_1 \perp L_2.
\end{equation}
Let $z_1^{(i)}$ and $z_2^{(i)}$ be the corresponding activations of $L_1$ and $L_2$ for sample $i$.  
In the ideal case (no splitting nor absorption, $\alpha = 0$), all information related to $F$ is represented solely by $L_1$, and $L_2$ contributes only to the orthogonal reconstruction components.  
The activation pattern is:
\[
\begingroup
\small
\centering
\begin{array}{c|cc}
\text{Sample} & L_1 & L_2 \\
\hline
1 & z_1^{(1)} & 0 \\
\vdots & \vdots & \vdots \\
m & z_1^{(m)} & 0 \\
m+1 & z_1^{(m+1)} & z_2^{(m+1)} \\
\vdots & \vdots & \vdots \\
m+n & z_1^{(m+n)} & z_2^{(m+n)} \\
\end{array}
\endgroup
\]
Here, $z_1^{(i)}$ corresponds to the feature-aligned activation along $L_1$, while $z_2^{(i)}$ encodes the orthogonal residual components required for accurate reconstruction. Samples $1, \dots, m$ exclusively activate $L_1$, whereas samples $m+1, \dots, m+n$ possess a non-zero component from $L_2$.
In practice, sparse autoencoders often learn perturbed latent directions, denoted $L_1'$ and $L_2'$, that deviate from the ideal basis.  
We assume $L_2'$ contains a fraction $\alpha \in [0,1]$ of the feature direction $L_1$, forming a new latent that partially overlaps with it:
\begin{equation}
L_1' = L_1, \quad L_2' = \frac{(1-\alpha)L_1 + \alpha L_2}{\|(1-\alpha)L_1 + \alpha L_2\|}.
\label{eq:latent_mixture_refined}
\end{equation}
Here $\alpha$ quantifies the degree of cross-latent feature sharing.  
When $\alpha = 0$, $L_2'$ is perfectly orthogonal to $L_1'$ (no dispersion).  
When $\alpha = 1$, $L_2'$ fully aligns with $L_1'$, corresponding to a complete feature split into two disjoint latents.

The corresponding activation pattern becomes:

\begin{figure*}[t]
    \centering
    \includegraphics[width=\linewidth]{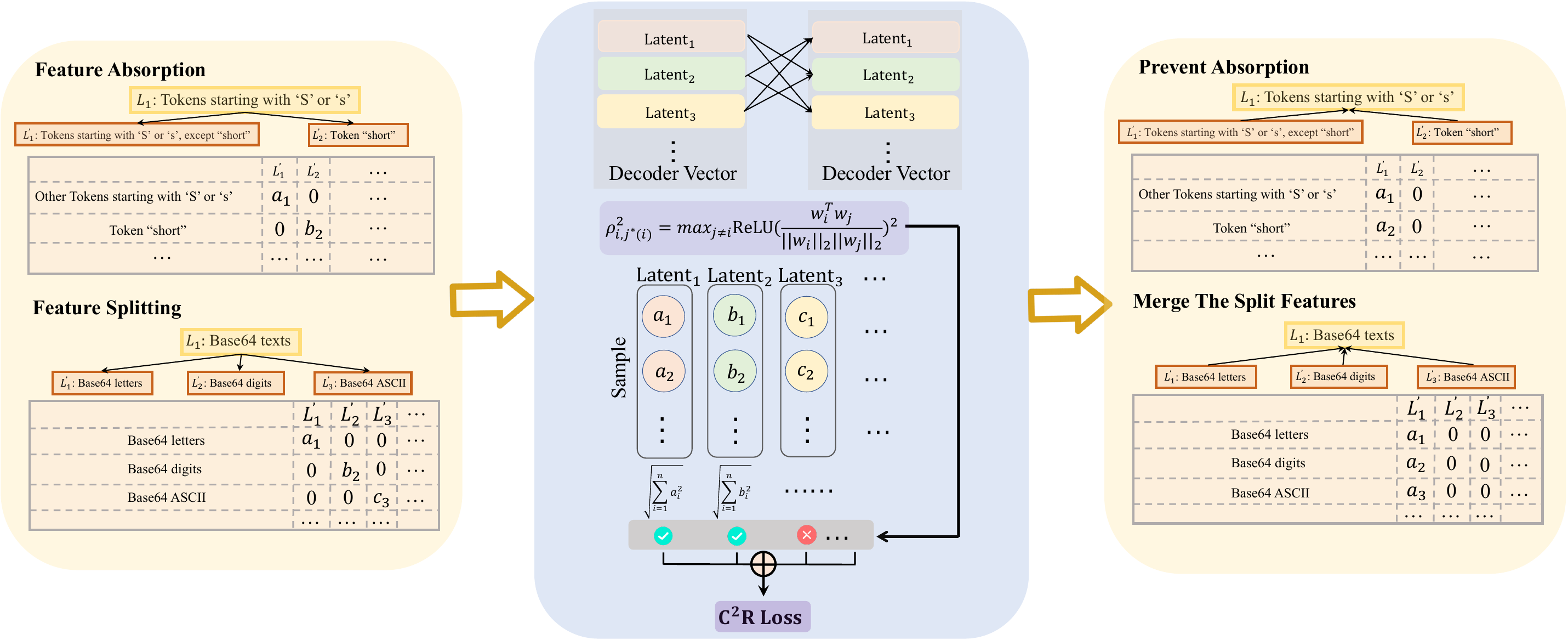}
\caption{Overview of C$^2$R (Cross-sample Consistency Regularization).  
Each mini-batch contains activations of multiple samples encoded by a sparse autoencoder.
C$^2$R enforces consistency of latent usage across samples by constraining activation patterns along the batch dimension.  
This encourages each latent to represent a complete semantic feature rather than fragmented or absorbed subfeatures, mitigating feature splitting and absorption.}
\label{fig:main_method}
\end{figure*}

\begin{table}[htbp]
  \centering 
  \scalebox{0.8}{
$\begin{array}{c|cc}
\text{Sample} & L_1' & L_2' \\
\hline
1 & z_1^{(1)} & 0 \\
\vdots & \vdots & \vdots \\
m & z_1^{(m)} & 0 \\
m+1 & (1-\alpha)z_1^{(m+1)} & \sqrt{(\alpha z_1^{(m+1)})^{2} + (z_2^{(m+1)})^{2}} \\
\vdots & \vdots & \vdots \\
m+n & (1-\alpha)z_1^{(m+n)} & \sqrt{(\alpha z_1^{(m+n)})^{2} + (z_2^{(m+n)})^{2}} \\
\end{array}$
}
\caption{Activation pattern when feature splitting or absorption occurs.}
\label{array:unified_absorption}
\end{table}
\vspace{-5pt}

This activation pattern illustrates that for samples $m+1$ through $m+n$, part of the original feature direction $L_1$ is reconstructed via $L_2'$ due to the shared component $\alpha L_1$ in Eq.~\ref{eq:latent_mixture_refined}.  
The $\sqrt{(\alpha z_1^{(i)})^{2} + (z_2^{(i)})^{2}}$ term ensures the total reconstructed model activation remains the same through vector addition of $L_1'$ and $L_2'$.  
This formalization unifies feature splitting and absorption:  
splitting and full absorption correspond to $\alpha$=$1$, where feature-aligned energy is duplicated across latents,  
and partial absorption corresponds to $\alpha$<$1$, where the $L_2'$ inherits part of the feature along with its orthogonal component.

\section{Theoretical Analysis on Two SAE Latents}
\label{sec:theory}

\subsection{Limitation of Per-sample Sparsity Objectives}
We analyze the behavior of $\ell_1$ and TopK objectives under the activation patterns defined in the unified problem formulation.

\begin{lemma}
Per-sample sparsity constraints, specifically $\ell_1$ regularization and TopK, strictly favor feature splitting and absorption ($\alpha \to 1$) over the ideal orthogonal decomposition ($\alpha = 0$) given equivalent reconstruction fidelity.
\end{lemma}

\begin{proof}
\noindent \textbf{Case 1: The $\ell_1$ Penalty.} 
We compare the cumulative $\ell_1$ penalty for the ideal configuration versus the perturbed configuration, assuming equivalent SAE reconstruction.

(1) \textbf{Ideal Decomposition ($\alpha=0$):} In this state, the feature $L_1$ and the residual $L_2$ are orthogonal. The total $\ell_1$ norm over the batch is:
\begin{equation}
\label{eq:l1_ideal}
\ell_1(\alpha=0) = \sum_{i=1}^{m+n} z_{1}^{(i)} + \sum_{i=m+1}^{m+n} z_{2}^{(i)}.
\end{equation}

(2) \textbf{Perturbed Decomposition ($\alpha>0$):} Under splitting or absorption, the activation energy is redistributed. The total $\ell_1$ norm becomes:
\begingroup
\normalsize
\begin{align}
\begin{split}
\label{eq:l1_absorb}
\ell_1(\alpha>0) &= \sum_{i=1}^{m} z_{1}^{(i)} + \sum_{i=m+1}^{m+n} (1-\alpha) z_{1}^{(i)} \\
&+ \sum_{i=m+1}^{m+n} \sqrt{(\alpha z_1^{(i)})^{2} + (z_2^{(i)})^{2}}.
\end{split}
\end{align}
\endgroup
We apply the Minkowski inequality to the terms summing over indices $i \in \{m+1, \dots, m+n\}$. For any $\alpha > 0$:
\begingroup
\small
\begin{equation}
\sum_{i=m+1}^{m+n} \alpha z_{1}^{(i)} + \sum_{i=m+1}^{m+n} z_{2}^{(i)} 
> \sum_{i=m+1}^{m+n} \sqrt{(\alpha z_1^{(i)})^{2} + (z_2^{(i)})^{2}}.
\end{equation}
\endgroup
Subtracting the shared terms from Eq.~\ref{eq:l1_ideal} and Eq.~\ref{eq:l1_absorb}, it follows that $\ell_1(\alpha>0) < \ell_1(\alpha=0)$. The strict inequality holds for all $\alpha \in (0, 1]$. Consequently, the optimization process minimizes the global objective by maximizing $\alpha$, thereby driving the solution toward splitting or absorption to reduce the total norm while maintaining reconstruction fidelity.

\noindent \textbf{Case 2: The TopK Constraint.} 
The TopK objective enforces a hard constraint on the number of active latents, effectively minimizing the $\ell_0$ norm of the activation vector for a fixed reconstruction error tolerance. We evaluate the cardinality of the active set for the intersection samples $i \in \{m+1, \dots, m+n\}$.

(1) \textbf{Ideal Decomposition ($\alpha=0$):} 
The signal consists of two orthogonal non-zero components: the feature activation $z_1^{(i)}$ and the residual $z_2^{(i)}$. Since the basis vectors $L_1$ and $L_2$ are orthogonal, exact representation requires both to be active. Thus, the sparsity consumption is:
\begin{equation}
\|z^{(i)}\|_0 = 2 \quad \text{for } i \in \{m+1, \dots, m+n\}.
\end{equation}

(2) \textbf{Full Splitting or Absorption ($\alpha=1$):} 
Substituting $\alpha=1$ into the activation pattern defined in Table 2, the coefficient for the first latent becomes zero: $(1-\alpha)z_1^{(i)} = 0$. The second latent, $L_2'$, captures the entire vector magnitude $\sqrt{(z_1^{(i)})^2 + (z_2^{(i)})^2}$. As a result, the SAE represents the same vector space using a single active latent:
\begin{equation}
\|z^{(i)}\|_0 = 1 \quad \text{for } i \in \{m+1, \dots, m+n\}.
\end{equation}

By reducing the active set from 2 latents to 1, the split configuration ($\alpha=1$) saves the sparsity budget. This creates a strong pressure on the optimization: the TopK constraint pushes the SAEs to use latents to represent mixed feature directions rather than maintaining atomic features, as this saves capacity within the $k$-latent budget to reconstruct other features and lower the global reconstruction loss.\end{proof}

\subsection{Mitigating Splitting via Cross-Sample Consistency}
Previous analysis shows that per-sample objectives ($\ell_1$ and TopK) cannot distinguish between atomic and split features. In contrast, the Minkowski inequality~\citep{Gruber1979} (Eq.\ref{eq:minkowski}) for $p=2$ offers a strict convexity condition to reverse this preference. Since the sum of norms for separate vectors is always greater than the norm of their sum, minimizing the sum of $\ell_2$ norms across the batch dimension, i.e. $\|Z_{:,1}\|_2 + \|Z_{:,2}\|_2$, encourages SAEs to consolidate shared semantic feature into a single latent.

We analyze this mechanism using the unified formulation in Table~\ref{array:unified_absorption}. We define a regularization term $\mathcal{L}_{\text{pair}}$ as the sum of the batch-norms for the two split latents:
\begingroup
\small
\begin{equation}
\begin{split}
\mathcal{L}_{\text{pair}}(\alpha) &= \|Z_{:,1}\|_2 + \|Z_{:,2}\|_2 \\
&= \sqrt{\sum_{i=1}^{m} (z_{1}^{(i)})^2 + \sum_{i=m+1}^{m+n} ((1-\alpha)z_{1}^{(i)})^2} \\
&+ \sqrt{\sum_{i=m+1}^{m+n} (\alpha z_1^{(i)})^2 + \sum_{i=m+1}^{m+n} (z_2^{(i)})^2}.
\end{split}
\end{equation}
\endgroup
To check if minimizing this term suppresses splitting, we calculate its gradient with respect to the splitting factor $\alpha$. The derivative $\frac{\partial \mathcal{L}_{\text{pair}}}{\partial \alpha}$ shows that the loss increases monotonically with $\alpha$ (meaning the penalty reduces splitting) as long as the following condition holds (derivation in Appendix~\ref{appx:C^2R_proof}):
\begin{equation}
\alpha \geq \frac{1}{\sqrt{\frac{\sum_{i=1}^{m} (z_{1}^{(i)})^{2}}{\sum_{i=m+1}^{m+n} (z_{2}^{(i)})^{2}}} + 1 }.
\label{eq:feature_disp_condition}
\end{equation}
Empirical results from recent literature support this condition for practical SAE training. \citet{leasksparse} find that split latents retain high cosine similarity with their parent latents, implying the feature component magnitude far exceeds the residual ($\|z_1^{(i)}\| \gg \|z_2^{(i)}\|$). Additionally, \citet{chanin2024absorption} note large frequency gaps between parent and child features (e.g., $P(f_0)=0.25$ vs $P(f_1)=0.05$), which means the cumulative energy of the primary feature dominates the residual:
\begin{equation}
\sum_{i=1}^{m} (z_{1}^{(i)})^{2} \gg \sum_{i=m+1}^{m+n} (z_{2}^{(i)})^{2}.
\label{eq:final_condition}
\end{equation}
Under these settings, the right-hand side of Eq.~\ref{eq:feature_disp_condition} approaches zero. As a result, $\mathcal{L}_{\text{pair}}$ increases monotonically with $\alpha$ in the relevant domain. This confirms that a cross-sample $\ell_2$ penalty on redundant pairs theoretically ensures the consolidation of semantically related activations into a single consistent latent, mitigating both splitting and absorption. 

We empirically verify the condition in Eq.~\ref{eq:feature_disp_condition} on our trained 65{,}536-latent SAEs using a 4M-token test set from SAEBench. As shown in Appendix~\ref{app:eq13_verification}, the condition is satisfied in 88.1\% of absorption pairs ($N{=}4{,}555$), and the median ratio of the left-hand side to the right-hand side of Eq.~\ref{eq:final_condition} is 81.56. The remaining 11.9\% of violated pairs correspond to marginal cases with low absorption coefficients, where neither the parent nor the child feature carries strong signal. This confirms that the theoretical guarantee holds in the vast majority of practically relevant cases, though we note its conditional nature.

\begin{figure*}[t] 
    \centering
    \begin{subfigure}[b]{0.24\textwidth}
        \centering
        \includegraphics[width=\linewidth]{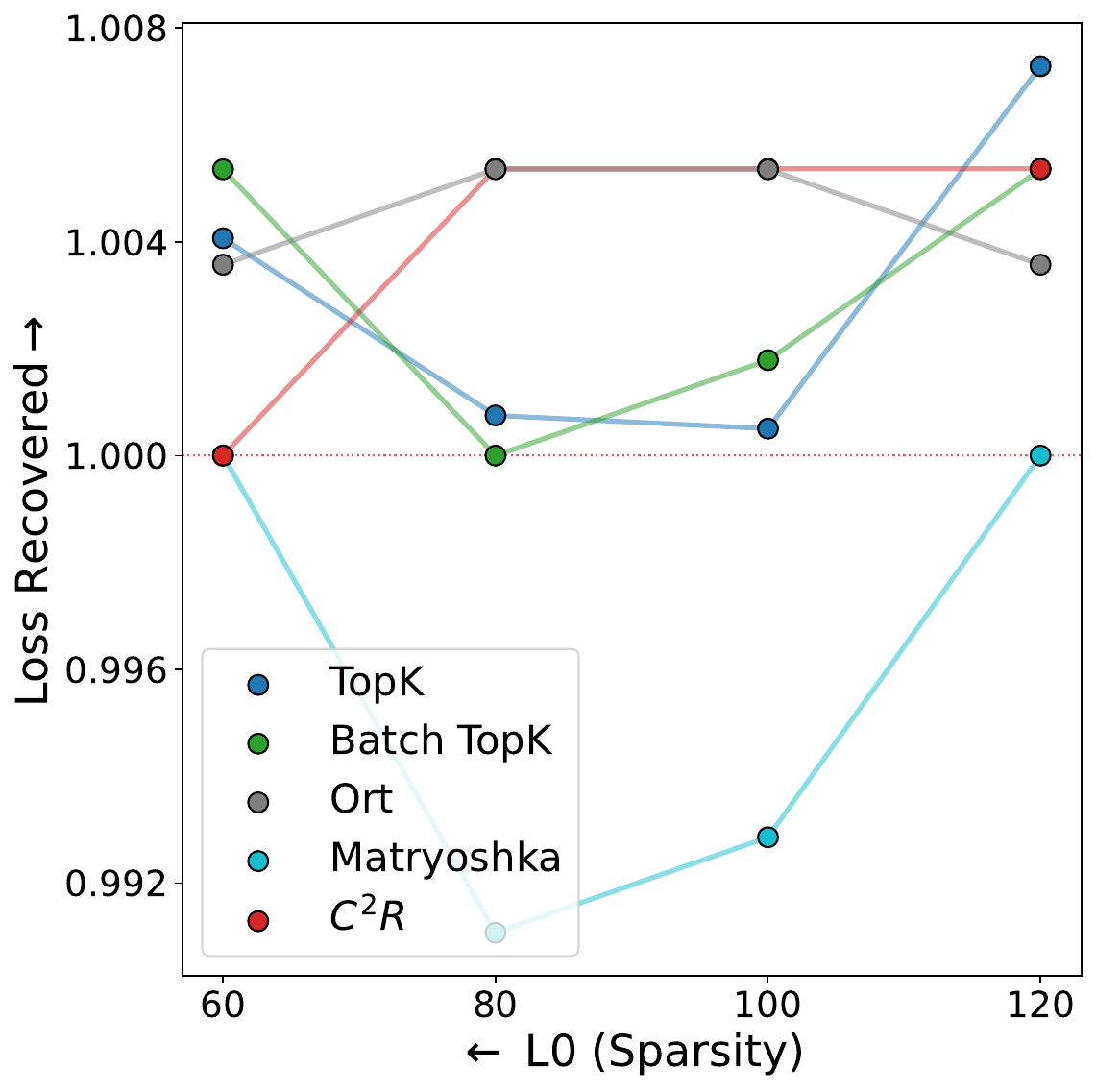} 
        \caption{Loss Recovered}
        \label{fig:main_experiment_a}
    \end{subfigure}
    \hfill 
    \begin{subfigure}[b]{0.24\textwidth}
        \centering
        \includegraphics[width=\linewidth]{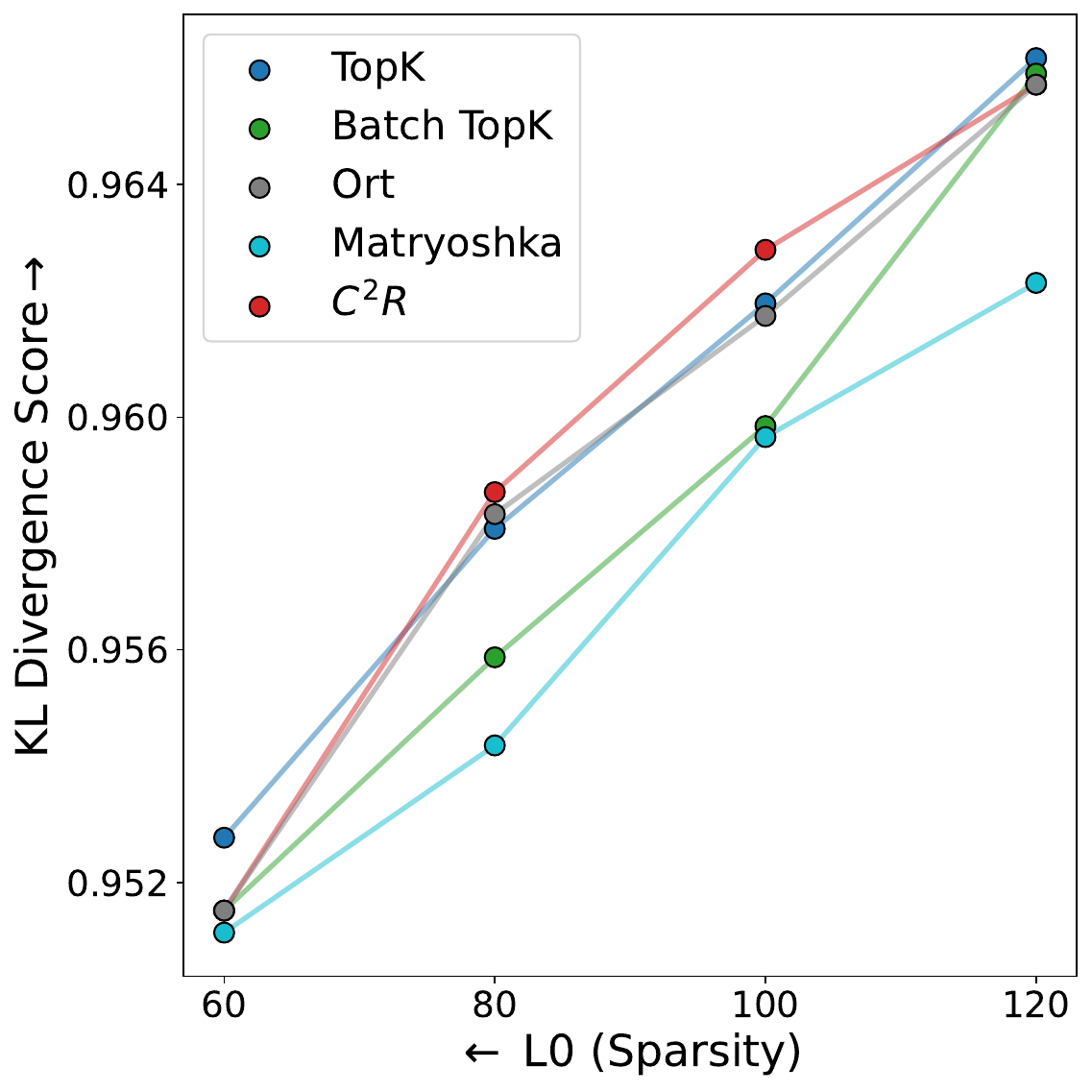}
        \caption{KL Divergence Score}
        \label{fig:main_experiment_b}
    \end{subfigure}
    \hfill
    \begin{subfigure}[b]{0.24\textwidth}
        \centering
        \includegraphics[width=\linewidth]{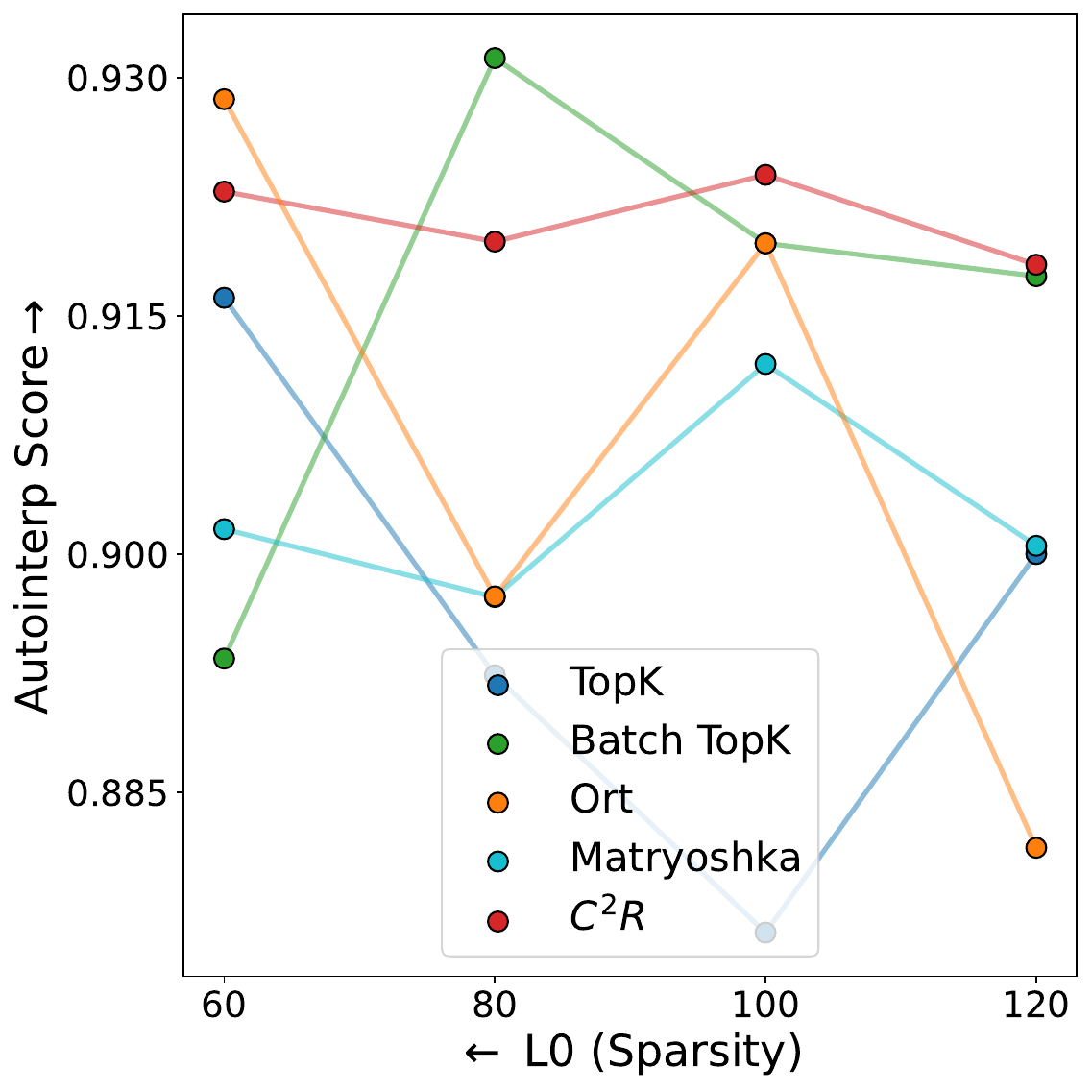}
        \caption{Autointerp}
        \label{fig:main_experiment_c}
    \end{subfigure}
    \hfill
    \begin{subfigure}[b]{0.24\textwidth}
        \centering
        \includegraphics[width=\linewidth]{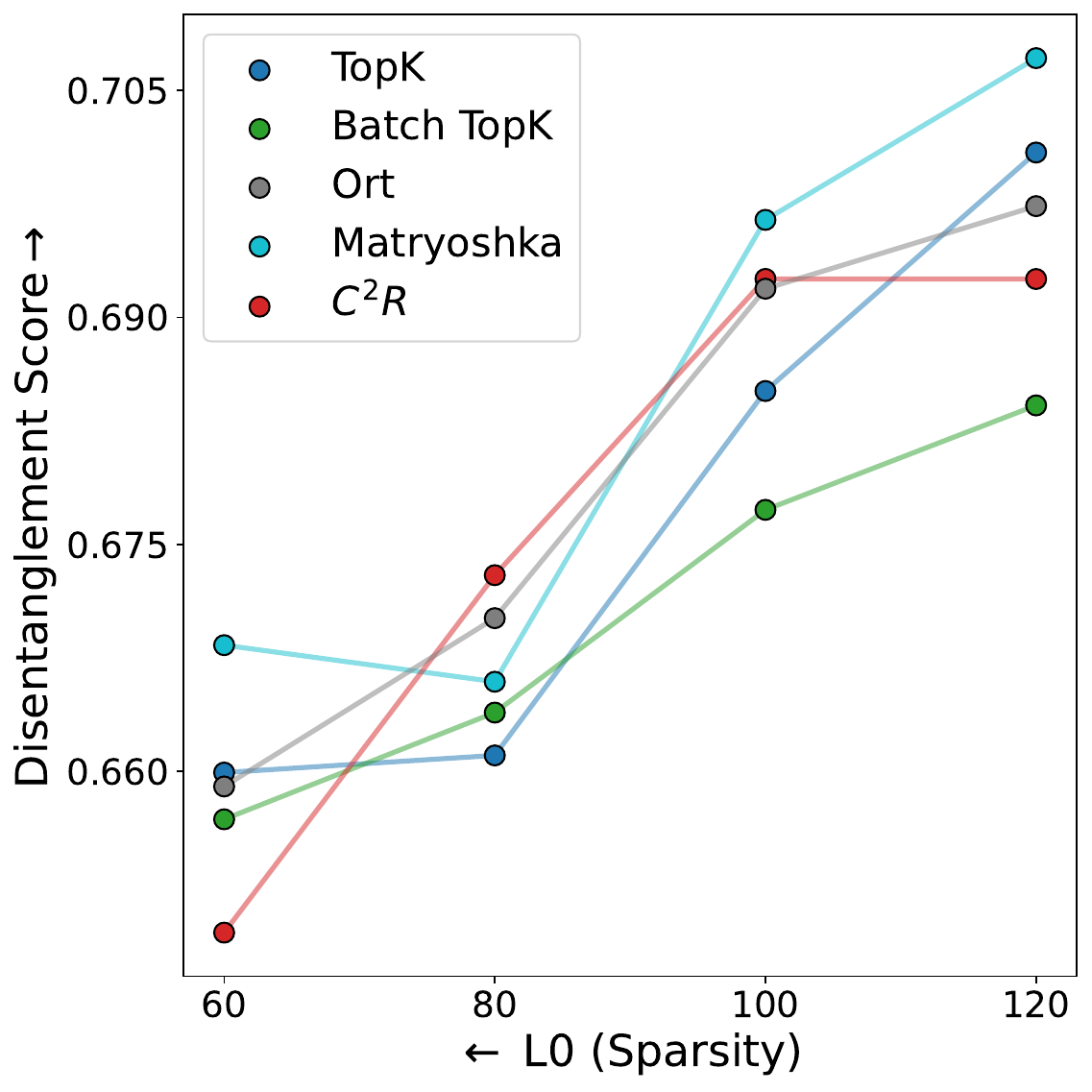}
        \caption{Disentanglement Score}
        \label{fig:main_experiment_d}
    \end{subfigure}
    
    \caption{\textbf{Quantitative comparison of SAE performance across different sparsity levels.} 
(a) and (b) evaluate \textbf{reconstruction fidelity} using Cross-Entropy Loss and KL Divergence at the LLM output layer, respectively, (c) assesses SAE latent \textbf{interpretability} using the Autointerp score, and (d) measures the extent to which real-world features are disentangled into steerable latents.
Notably, C$^2$R-enhanced Batch TopK SAE does not exhibit performance degradation compared to the vanilla Batch TopK baseline. Instead, it maintains competitive or superior results, particularly in KL Divergence and Autointerp scores, demonstrating that our proposed C$^2$R constraint preserves these important capabilities.}
    \label{fig:main_experiment}
\end{figure*}

\looseness=-1
\section{Cross-Sample Consistency Regularization}
Following the analysis in Section~\ref{sec:theory}, we introduce \textbf{C$^2$R} (\underline{\textbf{C}}ross-sample \underline{\textbf{C}}onsistency \underline{\textbf{R}}egularization). Minkowski inequality helps merge redundant features, but applying it to a large dictionary requires a careful approach. Here, we extend the pairwise analysis to the full SAE dictionary and examine the gradient dynamics that enable C$^2$R to enforce both cross-sample consistency and latent orthogonality.

\subsection{Generalizing from Pairwise to Multiple Latents}
The derivation in Section~\ref{sec:theory} used a simple system with two latents for one parent feature and a child feature. However, standard SAEs have thousands of latents, and most represent distinct concepts. If we minimize the sum of norms across random pairs without selection, we would force independent features to merge, which causes feature collapse and reduces the SAEs' ability to resolve semantic differences.

We therefore need the regularization to be selective. It should only penalize latent pairs that look like split fragments or absorbed variations, while leaving independent features alone. Previous work on feature splitting~\citep{chanin2024absorption} and orthogonal constraints~\citep{korznikov2025ort} shows that redundant latents usually have similar decoder weight directions. In contrast, distinct features tend to be nearly orthogonal in the high-dimensional space.

Based on this, we use the cosine similarity of decoder weights to detect redundancy. For each latent $i$, we find its nearest neighbor $j^*(i)$ in the decoder space:
\begin{equation}
    j^*(i) = \operatorname*{arg\,max}_{j \neq i} \, \langle \hat{w}_i, \hat{w}_j \rangle, \quad \text{where } \hat{w} = \frac{w}{\|w\|_2}.
\end{equation}
We then define the C$^2$R loss by weighting the pairwise norm penalty with the squared rectified cosine similarity $\rho_{i, j^*(i)}^2 = \operatorname{ReLU}(\langle \hat{w}_i, \hat{w}_{j^*(i)} \rangle)^2$:
\begin{equation}
    \mathcal{L}_{\text{C$^2$R}}(X) = \frac{1}{k} \sum_{i=1}^{k} \rho_{i, j^*(i)}^2 \cdot \underbrace{\left( \|Z_{:,i}\|_2 + \|Z_{:,j^*(i)}\|_2 \right)}_{S_{i, j^*(i)}}.
\end{equation}
This weight acts as a gate. For distinct features where $\hat{w}_i \perp \hat{w}_j$, $\rho^2$ is zero, so the consistency constraint does not apply. For redundant features with high alignment, $\rho^2$ is large, which fully applies the regularization.

The final training objective adds this regularization to the standard SAE loss, which includes reconstruction and sparsity terms:
\begin{equation}
    \mathcal{L}(X) = \mathcal{L}_{\text{SAE}}(X) + \lambda_{\text{C$^2$R}} \mathcal{L}_{\text{C$^2$R}}(X),
\end{equation}
where $\lambda_{\text{C$^2$R}}$ controls the weight of the cross-sample consistency term.

\subsection{Gradient Analysis and Implicit Orthogonality}
We can better understand C$^2$R by looking at its gradient dynamics. The loss function is the product of a geometric alignment term ($\rho^2$) and an activation magnitude term ($S_{i, j^*(i)}$). The gradient with respect to the model parameters $\theta$ follows the product rule $\nabla (AB) = B \nabla A + A \nabla B$:
\begin{equation}
    \nabla_\theta \mathcal{L}_{\text{C$^2$R}} \propto \underbrace{\rho^2 \cdot \nabla_\theta S_{i, j^*(i)}}_{\text{Consistency Gradient}} + \underbrace{S_{i, j^*(i)} \cdot \nabla_\theta (\rho^2)}_{\text{Orthogonality Gradient}}.
\end{equation}
This shows that C$^2$R creates two simultaneous forces during optimization.

\textbf{1. Consistency Pressure ($\rho^2 \nabla S$).} The first term minimizes the sum of norms, scaled by the similarity weight $\rho^2$. As shown in Section~\ref{sec:theory}, this uses the Minkowski inequality to push the splitting factor $\alpha$ toward 0. It merges the activation energy of redundant latents into a single one, which helps fix feature splitting and absorption.

\textbf{2. Implicit Orthogonality Pressure ($S \nabla \rho^2$).} The second term minimizes cosine similarity between decoder weights to encourage feature orthogonality, similar to the goals of OrtSAE~\citep{korznikov2025ort}. Unlike OrtSAE that typically applies a uniform penalty to all selected max-cosine pairs, our method scales the penalty by $S_{i, j^*(i)}$ (the sum of feature activations). This allows C$^2$R to adjust regularization based on feature frequency and magnitude. Strong, high-frequency features (large $S$) are subject to stricter orthogonality pressure to prevent redundancy. Conversely, for newly initialized or rare latents (small $S$), aggressive orthogonality enforcement can be counterproductive, potentially pushing them away from valid directions before they stabilize. By scaling the gradient with activation magnitude, C$^2$R prevents such disruption, allowing developing latents to converge naturally. This mechanism promotes orthogonality while adapting the regularization strength to each feature's convergence state.
We further clarify the implementation-level relationship between C$^2$R and OrtSAE in Appendix~\ref{appx:ortsae_relation}.

\begin{figure*}[t] 
    \centering
    \begin{subfigure}[b]{0.33\textwidth}
        \centering
        \includegraphics[width=\linewidth]{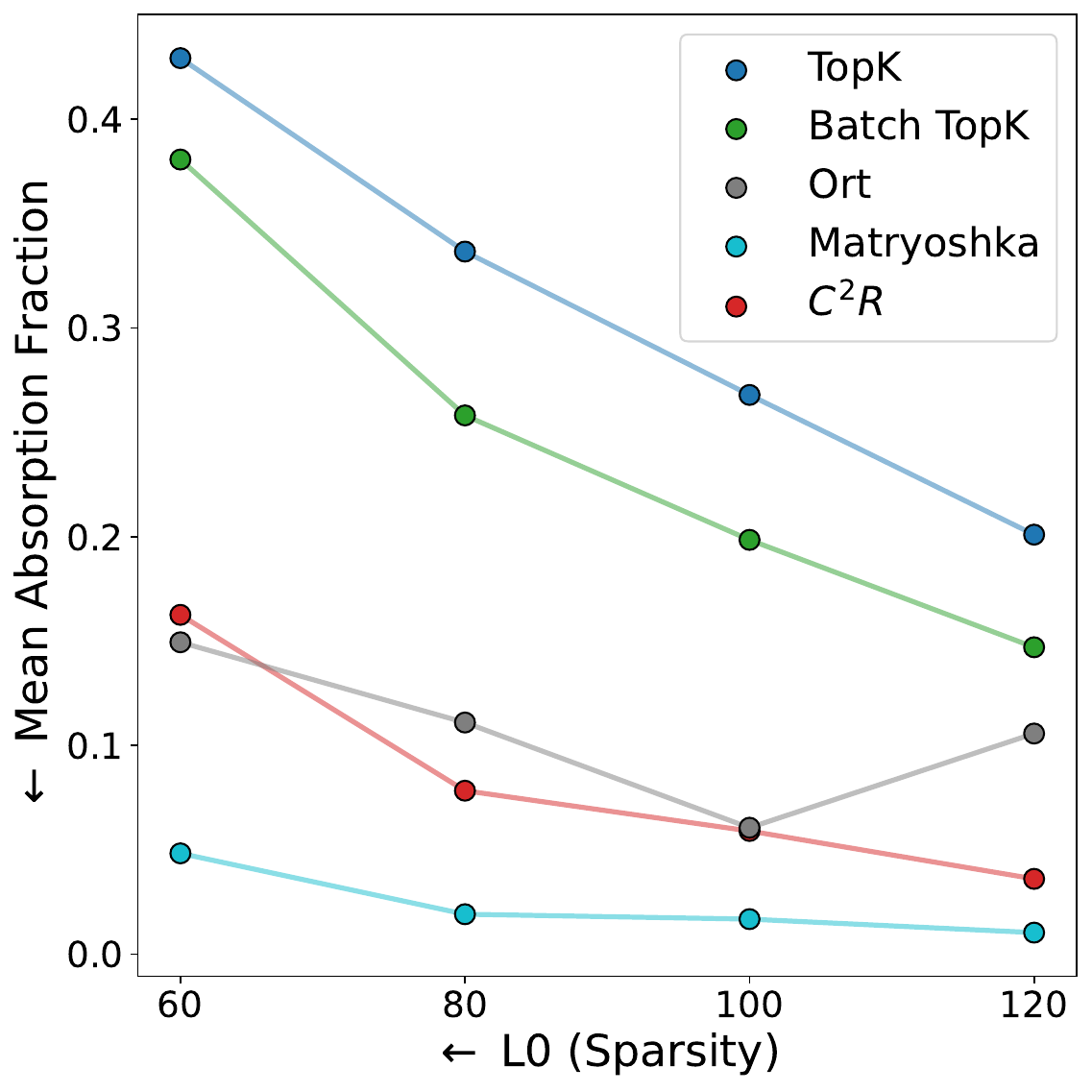} 
        \caption{Feature absorption}
        \label{fig:hierarchy_experiment_a}
    \end{subfigure}
    \hfill 
    \begin{subfigure}[b]{0.33\textwidth}
        \centering
        \includegraphics[width=\linewidth]{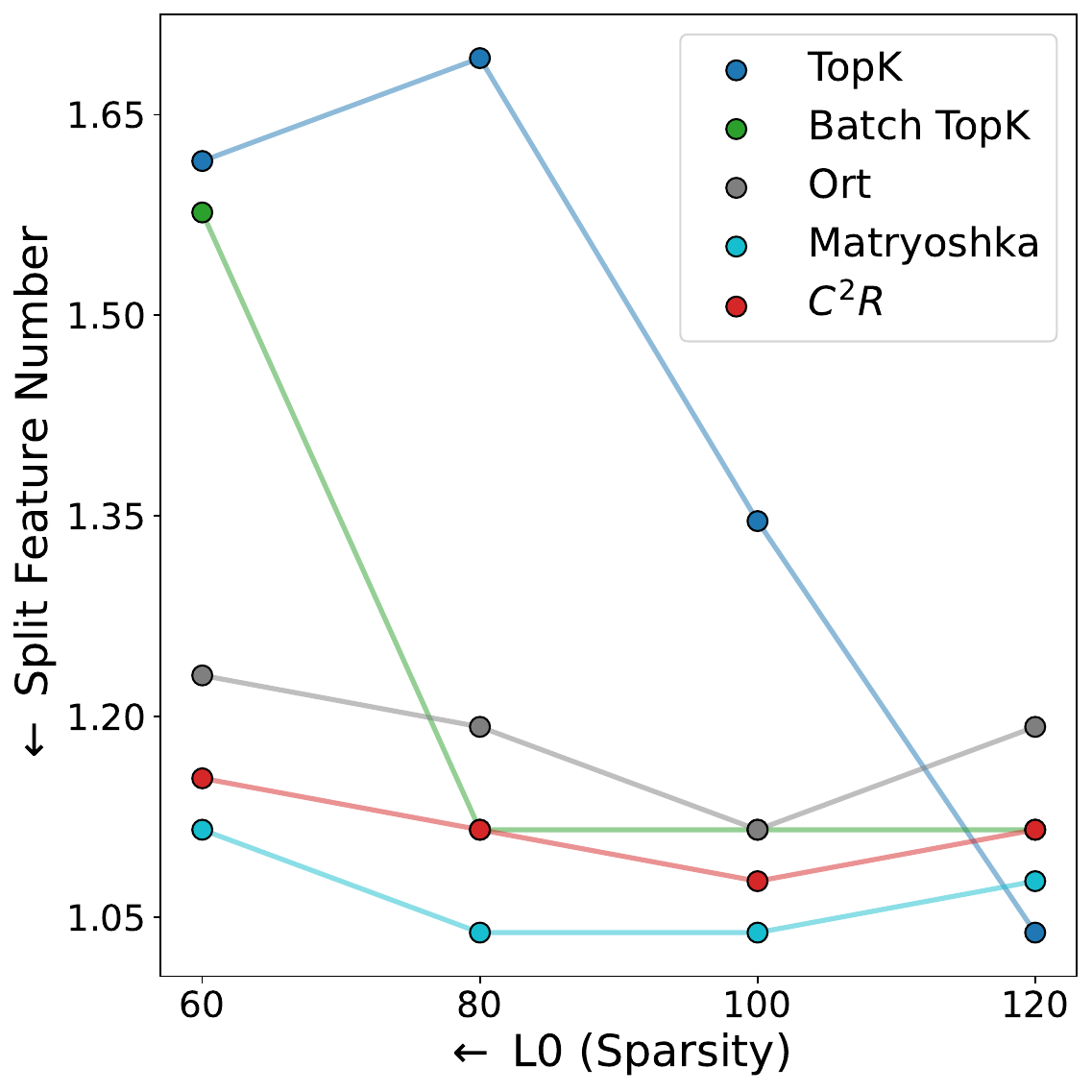}
        \caption{Feature splitting}
        \label{fig:hierarchy_experiment_b}
    \end{subfigure}
    \hfill
    \begin{subfigure}[b]{0.33\textwidth}
        \centering
        \includegraphics[width=\linewidth]{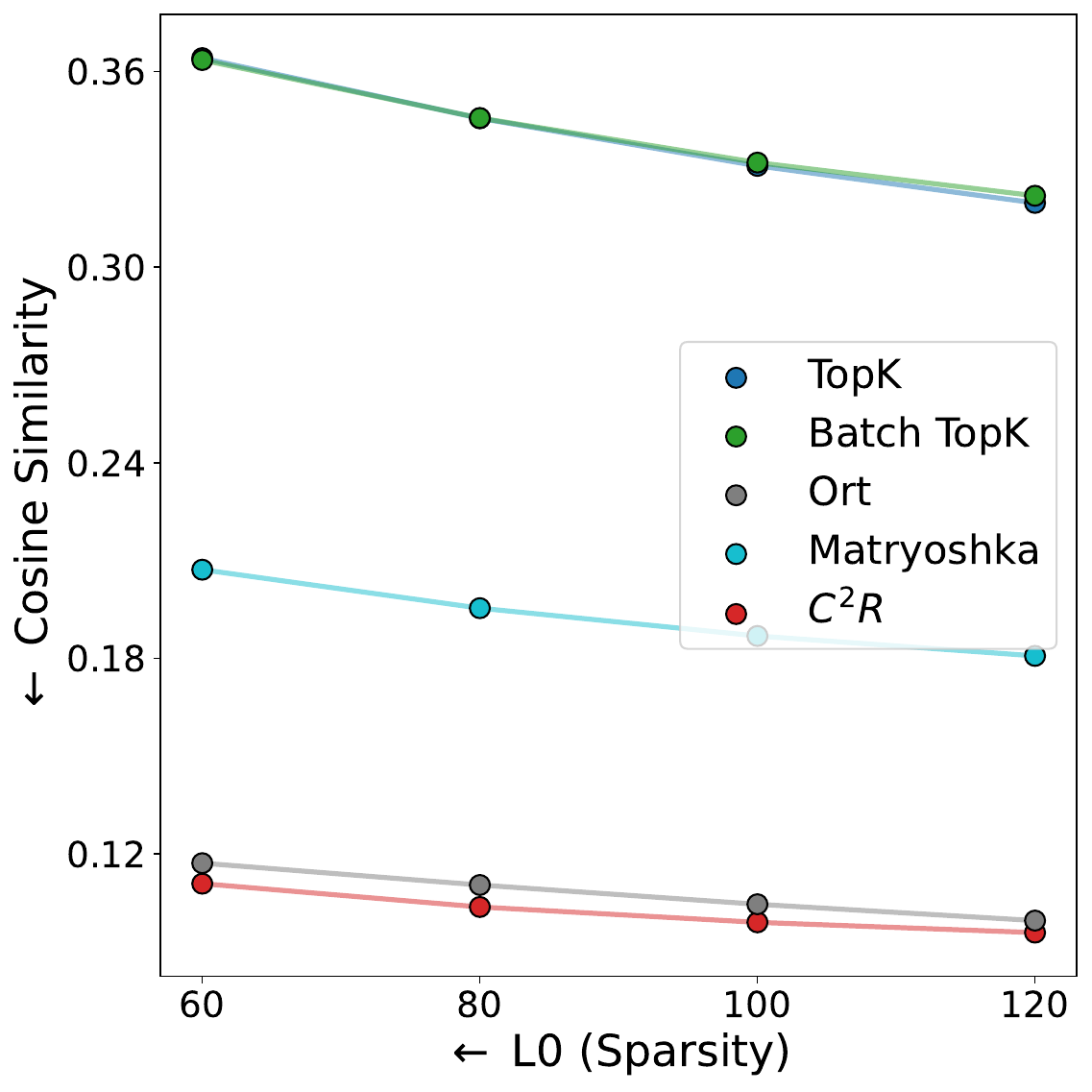}
        \caption{Feature composition}
        \label{fig:hierarchy_experiment_c}
    \end{subfigure}
    \caption{\textbf{Analysis of feature structural metrics.} We compare feature absorption, splitting, and composition across different sparsity levels. Among methods that maintain high reconstruction fidelity (i.e., excluding Matryoshka SAEs), the proposed C$^2$R constraint achieves the lowest rates of feature absorption and splitting. Furthermore, it achieves the optimal performance in feature decomposition, consistently yielding the lowest cosine similarity to ensure more atomic features.}
    \label{fig:hierarchy_experiment}
\end{figure*}

\section{EXPERIMENTS}
\label{sec:experiments}

\subsection{Baselines}
Baselines include four different SAE architectures: TopK SAEs~\citep{gaoscaling}, Batch TopK SAEs~\citep{leasksparse}, Matryoshka SAEs~\citep{bussmannlearning}, and OrtSAEs~\citep{korznikov2025ort}, as introduced in~\ref{sec:background}. For the main results in Figures~\ref{fig:main_experiment} and \ref{fig:hierarchy_experiment}, C$^2$R uses Batch TopK as its base architecture to ensure a fair comparison with OrtSAE and Matryoshka SAE, both of which also build on Batch TopK. We evaluate performance by iterating over sparsity levels of $k \in \{60,80,100,120\}$, ensuring a fair comparison between baselines and C$^2$R-enhanced Batch TopK SAE under equivalent sparsity conditions.\looseness=-1. We run baselines and implement C$^2$R on top of a public codebase~\citep{marks2024dictionary_learning}, and the trained SAEs are evaluated with SAEBench~\citep{karvonensaebench}. Matryoshka SAEs use 5 layers, with the sizes of its 4 sub-SAEs set to \{1/32, 1/8, 1/4, 1/2\} of the total latents. 

\subsection{Experiment Settings}
We conduct systematic experiments on \text{Gemma-2-2B}~\citep{gemma_2024}\footnote{Gemma Terms of Use} to validate the effectiveness of C$^2$R. Specifically, we sample a 500M-token subset from the \text{OpenWebText} dataset~\citep{Gokaslan2019OpenWeb}\footnote{Creative Commons Zero v1.0 Universal} and use it to train a series of SAEs on the residual stream activations of the 12th layer of \text{Gemma-2-2B}. Each SAE has 65,536 latents, which is approximately 28 times the model's residual dimension, making it easier to observe feature absorption and feature splitting~\citep{karvonensaebench}. We performed a hyperparameter sweep for $\lambda_{\text{C}^2\text{R}}$ over the set $\{0.1, 0.5, 1, 5, 10\}$. We selected $\lambda_{\text{C}^2\text{R}} = 5$, as it represents the maximal regularization strength that does not degrade reconstruction fidelity. All SAEs are optimized with \text{Adam} using a learning rate of $2 \times 10^{-4}$, batch size of 2,048, and context length of 1,024.

\textbf{Computational Efficiency.}\quad
Computing the pairwise cosine similarity for a large dictionary (e.g., $k=65,536$) imposes a quadratic $O(k^2)$ computational complexity and substantial memory overhead. To ensure computational feasibility and maintain a fair comparison with the state-of-the-art baseline, we adopt the efficient optimization strategy used in OrtSAEs~\citep{korznikov2025ort}. Specifically, we employ a block-wise computation strategy with a chunk size of 8,192 and compute the consistency regularization term every 5 training steps, scaling the coefficient $\lambda_{\text{C$^2$R}}$ accordingly. This approach reduces the overhead to negligible levels while preserving performance (detailed in Appendix~\ref{appx:efficiency}).

\subsection{Metrics}
\looseness=-1
To holistically assess the relative performance of SAEs when integrating C$^2$R, our evaluation utilizes a comprehensive set of seven key metrics. These metrics, implemented using the code framework from SAEBench~\citep{karvonensaebench}, cover four areas: \textbf{reconstruction fidelity}, \textbf{feature hierarchy}, \textbf{interpretability}, and \textbf{disentanglement}. Specifically, our evaluation consists of six key metrics: \textbf{Loss Recovered}, \textbf{KL Div. Score}, \textbf{AutoInterp}, \textbf{Split Num}, \textbf{Absorption Rate}, \textbf{Composition}, and \textbf{Disentanglement}. Detailed description of these metrics is in~\ref{appx:metrics_desc}.

\subsection{Reconstruction Fidelity}
As shown in Figure~\ref{fig:main_experiment}, integrating C$^2$R preserves reconstruction fidelity of the backbone SAEs. Figure~\ref{fig:main_experiment_a} reports the Loss Recovery metric (details in Appendix~\ref{appx:metrics_desc}), where a value $\ge 1$ indicates that the SAE-reconstructed LLM activations yield a cross-entropy loss lower than or equal to the original LLM. We observe that across the tested sparsity levels, only Matryoshka SAEs exhibit performance loss, while other methods maintain full recovery. Figure~\ref{fig:main_experiment_b} displays the KL Divergence Score, which measures the shift in the LLM's output logit distribution. Higher scores correspond to smaller deviations from the original distribution. The results demonstrate that SAEs trained with the C$^2$R constraint maintain high reconstruction fidelity, strictly adhering to the original model's behavior.

\subsection{Interpretability}

Following the evaluation framework of~\citep{pauloautomatically}, we assess the interpretability of all SAEs using the \textit{AutoInterp} metric. 
The results in Figure~\ref{fig:main_experiment_c} show that adding C$^2$R does not significantly affect interpretability. For each SAE, we sampled 128 latents and constructed prompts based on their activations over a 2M-token input. 
We prompt an LLM explainer to generate concise and comprehensive textual descriptions for each latent from 15 observed samples, and prompt an independent LLM judge predicted whether each latent would activate on 15 unseen test samples. 
The detailed prompts used for both the LLM explainer and LLM judger, as well as the instructions provided to human annotators, are included in Appendix~\ref{appx:prompts}.

We employ GPT-5-mini~\citep{gpt5mini2024} as the LLM judge and validate its reliability via a user study. GPT-5-mini achieved a 97.3\% match rate with humans and a Pearson $r$ of 0.74, confirming its suitability as a proxy evaluator for automated interpretability assessment. Details are in Appendix~\ref{appx:user_study}. To further verify that the results are not sensitive to the choice of judge model, we re-evaluate all AutoInterp scores using GPT-5. The per-method Pearson correlations between GPT-5-mini and GPT-5 scores range from 0.81 to 0.94 (all $p < 0.001$), confirming that the interpretability rankings are consistent across judge models. Full results are in Table~\ref{tab:judge_consistency}.

\subsection{Disentanglement}
We evaluate disentanglement using the RAVEL benchmark~\citep{huang-etal-2024-ravel}, which employs interchange interventions to determine if specific SAE latents causally control individual attributes such as continent or gender. The metric averages cause and isolation scores to measure how well the SAE governs a target concept without affecting unrelated ones. As illustrated in Figure~\ref{fig:main_experiment_d}, Batch TopK SAEs trained with the C$^2$R constraint achieve performance comparable to the vanilla baseline. This demonstrates that our regularization effectively preserves a disentangled and manipulable latent space.

\subsection{Feature Splitting and Feature Absorption}
\looseness=-1
Figures~\ref{fig:hierarchy_experiment_a} and~\ref{fig:hierarchy_experiment_b} illustrate the evaluation of feature absorption and splitting. Our proposed C$^2$R constraint effectively mitigates both issues compared to the vanilla Batch TopK baseline and outperforms the orthogonal constraint in OrtSAE. Although Matryoshka SAEs achieve lower absorption and splitting scores, this improvement comes at the expense of reconstruction fidelity, as evidenced in Figure~\ref{fig:main_experiment_a} and \ref{fig:main_experiment_b}. Therefore, among methods that maintain the original model's performance, the C$^2$R constraint strikes the best balance, achieving the lowest levels of absorption and splitting while preserving latent quality and reconstruction capabilities.

\subsection{Feature Composition}
Following~\citet{bussmannlearning}, we quantify feature composition by measuring the average maximum cosine similarity between latent vectors. High similarity implies that multiple latents represent overlapping information, indicating a lack of atomicity. As shown in Figure~\ref{fig:hierarchy_experiment_c}, our approach achieves the optimal performance on this metric. The C$^2$R constraint yields significantly lower cosine similarity compared to vanilla Batch TopK and Matryoshka SAEs, and marginally outperforms OrtSAE. This result confirms that our method effectively minimizes redundancy, producing a set of highly distinct and atomic features.

\subsection{Robustness and Generalization}\label{sec:robustness}

We conduct a series of additional experiments to verify the robustness and generalization of C$^2$R across random seeds, model scales, layers, and training data. Detailed results are provided in Appendix~\ref{app:robustness}.

\textbf{Statistical validation.}\quad
To assess the stability of our main results, we perform 5 independent training runs with different random seeds at sparsity level $k{=}100$ on Gemma-2-2B layer 12. As shown in Table~\ref{tab:stat_valid}, C$^2$R consistently achieves the best feature composition and absorption with narrow confidence intervals, confirming that the reported improvements are statistically robust.

\textbf{Scaling to larger models.}\quad
We evaluate C$^2$R on two current-generation models, Qwen3-8B and Llama-3-8B, both at layer 20 with sparsity level $k{=}100$. Tables~\ref{tab:qwen3} and \ref{tab:llama3} show that C$^2$R maintains its effectiveness at the 8B scale: it achieves the lowest composition and absorption while preserving reconstruction fidelity, demonstrating that our method generalizes beyond the Gemma-2-2B setting.

\textbf{Cross-layer consistency.}\quad
Beyond the layer 12 experiments in the main evaluation, we additionally train SAEs on Gemma-2-2B layer 20 (Table~\ref{tab:gemma_l20}). C$^2$R continues to achieve optimal or near-optimal structural metrics at a deeper layer, confirming cross-layer consistency.

\textbf{Sensitivity to training data.}\quad
We investigate the effect of training data scale and composition by (1) extending the training corpus from 500M to 1B tokens on OpenWebText, and (2) training on The Pile~\citep{gao2020pile}, a structurally diverse dataset spanning academic, internet, prose, and dialogue domains. As shown in Tables~\ref{tab:1b_owt} and \ref{tab:pile}, C$^2$R consistently achieves optimal or near-optimal structural feature metrics in both settings, confirming that our method is not sensitive to the scale or composition of the training data.

\subsection{Ablation Study}\label{sec:ablation}

We ablate the key design choices of C$^2$R and study the sensitivity to the regularization strength $\lambda_{\text{C}^2\text{R}}$. All experiments are conducted on Gemma-2-2B layer 12 at sparsity level $k{=}100$. Detailed results are provided in Appendix~\ref{app:ablation}.

\textbf{Component ablation.}\quad
We evaluate two variants: (1) \textbf{NoNNR}, which removes the nearest-neighbor restriction and applies the consistency penalty to all latent pairs; and (2) \textbf{NoRCG}, which removes the ReLU cosine gate and applies the penalty without directional selectivity. As shown in Table~\ref{tab:ablation}, removing either component imposes an overly aggressive merging constraint. While this severely penalizes feature absorption, it does so at the unacceptable cost of significantly degrading reconstruction fidelity and inflating composition scores. This demonstrates that the cautious, gated constraints of C$^2$R are necessary to mitigate feature pathologies without destroying dictionary utility.

\textbf{Hyperparameter sensitivity.}\quad
We sweep $\lambda_{\text{C}^2\text{R}} \in \{0.1, 0.5, 1, 5, 10\}$ to understand its effect on the trade-off between reconstruction and structural metrics (Table~\ref{tab:lambda}). As $\lambda_{\text{C}^2\text{R}}$ increases, composition and absorption decrease monotonically, while reconstruction fidelity remains stable up to $\lambda_{\text{C}^2\text{R}}{=}5$ and begins to degrade at $\lambda_{\text{C}^2\text{R}}{=}10$. We select $\lambda_{\text{C}^2\text{R}}{=}5$ as the default, as it maximizes the reduction in feature composition and absorption without sacrificing reconstruction quality.

\section{Compatibility with Other SAE Architectures}\label{sec:compatibility}

A key design goal of C$^2$R is to serve as a general-purpose regularizer that can be applied on top of any SAE architecture. To verify this, we evaluate C$^2$R when integrated with different backbone architectures beyond the Batch TopK used in our main experiments. All experiments are conducted on Gemma-2-2B layer 12 at sparsity level $k{=}100$. Detailed results are provided in Appendix~\ref{app:compatibility}.

\textbf{Integration with TopK and OrtSAE.}\quad
We apply C$^2$R to both the standard TopK and OrtSAE backbones. As shown in Table~\ref{tab:backbone}, C$^2$R consistently and substantially improves feature structural metrics regardless of the base architecture, without degrading reconstruction fidelity. Notably, TopK + C$^2$R achieves a significant reduction in Composition and Absorption, demonstrating that C$^2$R's benefits are not tied to the Batch TopK backbone but arise from the cross-sample consistency mechanism itself.

\textbf{Integration with AbsTopK.}\quad
To evaluate whether C$^2$R adapts to bi-directional encoder architectures, we apply it on top of AbsTopK, which uses signed activations. As shown in Table~\ref{tab:backbone}, C$^2$R effectively improves feature hierarchy in this signed setting while preserving reconstruction fidelity. The ReLU cosine gate remains appropriate because feature absorption and splitting imply cosine similarity $> 0$ between the involved decoder directions. By targeting only positive similarities, C$^2$R avoids collateral damage to non-pathological negative correlations, ensuring training stability in signed encoder settings.

\subsection{Downstream Causal Intervention Tasks}\label{sec:downstream}

We further evaluate the practical utility of learned SAE features on two downstream causal intervention tasks from SAEBench. Spurious Correlation Removal (SCR) measures the ability to remove spurious correlations by ablating relevant SAE features, and Targeted Probe Perturbation (TPP) measures the precision of causal interventions via targeted feature perturbation. As shown in Table~\ref{tab:downstream} in Appendix~\ref{app:downstream}, C$^2$R outperforms TopK, Batch TopK, and Ort on both tasks and achieves comparable performance to Matryoshka, indicating that the improved feature structure transfers to downstream causal applications.

\section{Conclusion}
In this work, we introduced C$^2$R, a theoretically grounded constraint aimed at improving feature hierarchy in sparse autoencoders. By leveraging the Minkowski inequality, our approach provides a rigorous guarantee against feature splitting and absorption. Extensive empirical evaluations show that incorporating the C$^2$R constraint significantly enhances the structural quality of the learned dictionary, reducing feature composition and minimizing redundancy. Importantly, our approach preserves reconstruction fidelity, latent interpretability, and disentanglement capabilities on par with the baselines. These results highlight the effectiveness of the cross-sample consistency regularization in addressing the trade-off between feature atomicity and model fidelity, offering a robust solution for large language model analysis.

\section{Limitations}
The theoretical guarantee in Eq.~\ref{eq:feature_disp_condition} is conditional and holds in 88.1\% of absorption pairs in our verification (Appendix~\ref{app:eq13_verification}). A fully unconditional guarantee remains an open problem. Although we extend our evaluation to Qwen3-8B and Llama-3-8B, we have not verified scalability to larger models or diverse architectures such as Mixture-of-Experts. Our AutoInterp evaluation uses the proprietary GPT-5-mini as the judge model. We verify high consistency with GPT-5 and human evaluations, but this still introduces a reproducibility constraint. In addition, C$^2$R may suppress legitimate polysemanticity when distinct features have moderately aligned decoder directions. We find that C$^2$R maintains over 99\% dictionary utilization (Appendix~\ref{appx:dead_features}), but the interaction with polysemantic features needs further study.

\section*{Acknowledgements}
This work was supported by the National Natural Science Foundation of China (NSFC) (No. U24A20253, NO. 62476279, NO. 92470205, NO. U2436209), Scientific Research Innovation Capability Support Project for Young Faculty, Major Innovation \& Planning Interdisciplinary Platform for the “Double-First Class” Initiative, Renmin University of China, the Fundamental Research Funds for the Central Universities, and the Research Funds of Renmin University of China No. 24XNKJ18. Supported by fund for building world-class universities (disciplines) of Renmin University of China and Public Computing Cloud, Renmin University of China.

\section*{Impact Statement}
This work advances the methodology of sparse dictionary learning for neural networks. By improving the fidelity of feature extraction, we aim to enable a more granular understanding of large language model internals. Such interpretability is essential for auditing model behavior, identifying latent failure modes, and verifying safety properties prior to deployment. However, we acknowledge that deeper insights into model representations can also be leveraged to improve model efficiency or steerability, potentially accelerating the development of powerful systems and amplifying the societal risks associated with their deployment.


\bibliography{example_paper}

@misc{korznikov2025ort,
      title={OrtSAE: Orthogonal Sparse Autoencoders Uncover Atomic Features}, 
      author={Anton Korznikov and Andrey Galichin and Alexey Dontsov and Oleg Rogov and Elena Tutubalina and Ivan Oseledets},
      year={2025},
      eprint={2509.22033},
      archivePrefix={arXiv},
      primaryClass={cs.LG},
      url={https://arxiv.org/abs/2509.22033}, 
}

@misc{Gokaslan2019OpenWeb,
    title={OpenWebText Corpus},
    author={Gokaslan, Aaron and Cohen, Vanya and Pavlick, Ellie and Tellex, Stefanie},
    howpublished={\url{http://Skylion007.github.io/OpenWebTextCorpus}},
    year={2019}
}

@article{gao2020pile,
  title={The pile: An 800gb dataset of diverse text for language modeling},
  author={Gao, Leo and Biderman, Stella and Black, Sid and Golding, Laurence and Hoppe, Travis and Foster, Charles and Phang, Jason and He, Horace and Thite, Anish and Nabeshima, Noa and others},
  journal={arXiv preprint arXiv:2101.00027},
  year={2020}
}

@inproceedings{huang-etal-2024-ravel,
    title = "{RAVEL}: Evaluating Interpretability Methods on Disentangling Language Model Representations",
    author = "Huang, Jing  and
      Wu, Zhengxuan  and
      Potts, Christopher  and
      Geva, Mor  and
      Geiger, Atticus",
    editor = "Ku, Lun-Wei  and
      Martins, Andre  and
      Srikumar, Vivek",
    booktitle = "Proceedings of the 62nd Annual Meeting of the Association for Computational Linguistics (Volume 1: Long Papers)",
    month = aug,
    year = "2024",
    address = "Bangkok, Thailand",
    publisher = "Association for Computational Linguistics",
    url = "https://aclanthology.org/2024.acl-long.470/",
    doi = "10.18653/v1/2024.acl-long.470",
    pages = "8669--8687"
}

@inproceedings{huben2023sparse,
  title={Sparse autoencoders find highly interpretable features in language models},
  author={Huben, Robert and Cunningham, Hoagy and Smith, Logan Riggs and Ewart, Aidan and Sharkey, Lee},
  booktitle={The Twelfth International Conference on Learning Representations},
  year={2023}
}

@article{bricken2023monosemanticity,
   title={Towards Monosemanticity: Decomposing Language Models With Dictionary Learning},
   author={Bricken, Trenton and Templeton, Adly and Batson, Joshua and Chen, Brian and Jermyn, Adam and Conerly, Tom and Turner, Nick and Anil, Cem and Denison, Carson and Askell, Amanda and Lasenby, Robert and Wu, Yifan and Kravec, Shauna and Schiefer, Nicholas and Maxwell, Tim and Joseph, Nicholas and Hatfield-Dodds, Zac and Tamkin, Alex and Nguyen, Karina and McLean, Brayden and Burke, Josiah E and Hume, Tristan and Carter, Shan and Henighan, Tom and Olah, Christopher},
   year={2023},
   journal={Transformer Circuits Thread},
   note={https://transformer-circuits.pub/2023/monosemantic-features/index.html}
}

@article{rajamanoharan2024improving,
  title={Improving dictionary learning with gated sparse autoencoders},
  author={Rajamanoharan, Senthooran and Conmy, Arthur and Smith, Lewis and Lieberum, Tom and Varma, Vikrant and Kram{\'a}r, J{\'a}nos and Shah, Rohin and Nanda, Neel},
  journal={arXiv preprint arXiv:2404.16014},
  year={2024}
}

@article{rajamanoharan2024jumping,
  title={Jumping ahead: Improving reconstruction fidelity with jumprelu sparse autoencoders},
  author={Rajamanoharan, Senthooran and Lieberum, Tom and Sonnerat, Nicolas and Conmy, Arthur and Varma, Vikrant and Kram{\'a}r, J{\'a}nos and Nanda, Neel},
  journal={arXiv preprint arXiv:2407.14435},
  year={2024}
}

@inproceedings{gaoscaling,
  title={Scaling and evaluating sparse autoencoders},
  author={Gao, Leo and la Tour, Tom Dupre and Tillman, Henk and Goh, Gabriel and Troll, Rajan and Radford, Alec and Sutskever, Ilya and Leike, Jan and Wu, Jeffrey},
  booktitle={The Thirteenth International Conference on Learning Representations}
}

@article{olshausen1997sparse,
  title={Sparse coding with an overcomplete basis set: A strategy employed by V1?},
  author={Olshausen, Bruno A and Field, David J},
  journal={Vision research},
  volume={37},
  number={23},
  pages={3311--3325},
  year={1997},
  publisher={Elsevier}
}

@inproceedings{ferrandoknow,
  title={Do I Know This Entity? Knowledge Awareness and Hallucinations in Language Models},
  author={Ferrando, Javier and Obeso, Oscar Balcells and Rajamanoharan, Senthooran and Nanda, Neel},
  booktitle={The Thirteenth International Conference on Learning Representations}
}

@article{chanin2024absorption,
  title={A is for absorption: Studying feature splitting and absorption in sparse autoencoders},
  author={Chanin, David and Wilken-Smith, James and Dulka, Tom{\'a}{\v{s}} and Bhatnagar, Hardik and Golechha, Satvik and Bloom, Joseph},
  journal={arXiv preprint arXiv:2409.14507},
  year={2024}
}

@inproceedings{bussmannlearning,
  title={Learning Multi-Level Features with Matryoshka Sparse Autoencoders},
  author={Bussmann, Bart and Nabeshima, Noa and Karvonen, Adam and Nanda, Neel},
  booktitle={Forty-second International Conference on Machine Learning}
}

@inproceedings{leasksparse,
  title={Sparse Autoencoders Do Not Find Canonical Units of Analysis},
  author={Leask, Patrick and Bussmann, Bart and Pearce, Michael T and Bloom, Joseph Isaac and Tigges, Curt and Al Moubayed, Noura and Sharkey, Lee and Nanda, Neel},
  booktitle={The Thirteenth International Conference on Learning Representations}
}

@inproceedings{zhao-etal-2025-steering,
    title = "Steering Knowledge Selection Behaviours in {LLM}s via {SAE}-Based Representation Engineering",
    author = "Zhao, Yu  and
      Devoto, Alessio  and
      Hong, Giwon  and
      Du, Xiaotang  and
      Gema, Aryo Pradipta  and
      Wang, Hongru  and
      He, Xuanli  and
      Wong, Kam-Fai  and
      Minervini, Pasquale",
    editor = "Chiruzzo, Luis  and
      Ritter, Alan  and
      Wang, Lu",
    booktitle = "Proceedings of the 2025 Conference of the Nations of the Americas Chapter of the Association for Computational Linguistics: Human Language Technologies (Volume 1: Long Papers)",
    month = apr,
    year = "2025",
    address = "Albuquerque, New Mexico",
    publisher = "Association for Computational Linguistics",
    url = "https://aclanthology.org/2025.naacl-long.264/",
    doi = "10.18653/v1/2025.naacl-long.264",
    pages = "5117--5136",
    ISBN = "979-8-89176-189-6"
}

@article{yeo2025understanding,
  title={Understanding Refusal in Language Models with Sparse Autoencoders},
  author={Yeo, Wei Jie and Prakash, Nirmalendu and Neo, Clement and Lee, Roy Ka-Wei and Cambria, Erik and Satapathy, Ranjan},
  journal={arXiv preprint arXiv:2505.23556},
  year={2025}
}

@article{wang2025resa,
  title={Resa: Transparent Reasoning Models via SAEs},
  author={Wang, Shangshang and Asilis, Julian and Akg{\"u}l, {\"O}mer Faruk and Bilgin, Enes Burak and Liu, Ollie and Fu, Deqing and Neiswanger, Willie},
  journal={arXiv preprint arXiv:2506.09967},
  year={2025}
}

@misc{lan2025quantifyingfeaturespaceuniversality,
      title={Quantifying Feature Space Universality Across Large Language Models via Sparse Autoencoders}, 
      author={Michael Lan and Philip Torr and Austin Meek and Ashkan Khakzar and David Krueger and Fazl Barez},
      year={2025},
      eprint={2410.06981},
      archivePrefix={arXiv},
      primaryClass={cs.LG},
      url={https://arxiv.org/abs/2410.06981}, 
}

@article{gemma_2024,
    title={Gemma},
    url={https://www.kaggle.com/m/3301},
    DOI={10.34740/KAGGLE/M/3301},
    publisher={Kaggle},
    author={Gemma Team},
    year={2024}
}

@misc{marks2024dictionary_learning,
   title = {dictionary\_learning\_demo},
   author = {Adam Karvonen},
   year = {2024},
   howpublished = {\url{https://github.com/adamkarvonen/dictionary_learning_demo}},
}

@inproceedings{karvonensaebench,
  title={SAEBench: A Comprehensive Benchmark for Sparse Autoencoders in Language Model Interpretability},
  author={Karvonen, Adam and Rager, Can and Lin, Johnny and Tigges, Curt and Bloom, Joseph Isaac and Chanin, David and Lau, Yeu-Tong and Farrell, Eoin and McDougall, Callum Stuart and Ayonrinde, Kola and others},
  booktitle={Forty-second International Conference on Machine Learning}
}

@inproceedings{pauloautomatically,
  title={Automatically Interpreting Millions of Features in Large Language Models},
  author={Paulo, Gon{\c{c}}alo Santos and Mallen, Alex Troy and Juang, Caden and Belrose, Nora},
  booktitle={Forty-second International Conference on Machine Learning}
}

@online{gpt5mini2024,
  author = {OpenAI},
  title = {GPT-5 System Card},
  year = {2025},
  url = {https://openai.com/index/gpt-5-system-card/},
  urldate = {2025-09-30}
}

@article{pearson1895vii,
  title={VII. Note on regression and inheritance in the case of two parents},
  author={Pearson, Karl},
  journal={proceedings of the royal society of London},
  volume={58},
  number={347-352},
  pages={240--242},
  year={1895},
  publisher={The Royal Society London}
}

@Inbook{Gruber1979,
author="Gruber, Peter M.",
editor="T{\"o}lke, J{\"u}rgen
and Wills, J{\"o}rg M.",
title="Geometry of numbers",
bookTitle="Contributions to Geometry: Proceedings of the Geometry-Symposium held in Siegen June 28, 1978 to July 1, 1978",
year="1979",
publisher="Birkh{\"a}user Basel",
address="Basel",
pages="186--225",
isbn="978-3-0348-5765-9",
doi="10.1007/978-3-0348-5765-9_7",
url="https://doi.org/10.1007/978-3-0348-5765-9_7"
}
\bibliographystyle{icml2026}

\newpage

\appendix
\onecolumn

\section{C$^{2}$R Promotes Cross-sample Consistency}
\label{appx:C^2R_proof}
When feature absorption occurs, the C$^{2}$R Loss is: 
\begingroup
\small
\begin{equation}
\begin{split}
\mathcal{L}_{\text{C$^2$R}}(X) &=
\sqrt{\sum_{i=1}^{m} (z_{1}^{(i)})^2 + \sum_{i=m+1}^{m+n} ((1-\alpha)z_{1}^{(i)})^2} \\
&+ \sqrt{\sum_{i=m+1}^{m+n} (\alpha z_1^{(i)})^2 + \sum_{i=m+1}^{m+n} (z_2^{(i)})^2}.
\end{split}
\end{equation}
\endgroup
Taking the partial derivative with respect to $\alpha$ yields:
\begingroup
\small
\begin{equation}
\begin{split}
\frac{\partial \mathcal{L}_{\text{C$^2$R}}(X)}{\partial \alpha} &=\frac{(\alpha-1)\sum_{i=m+1}^{m+n}(z_{1}^{(i)})^{2}}{
\sqrt{\sum_{i=1}^{m} (z_{1}^{(i)})^2 + \sum_{i=m+1}^{m+n} ((1-\alpha)z_{1}^{(i)})^2}} \\
&+ \frac{\alpha\sum_{i=m+1}^{m+n}(z_{1}^{(i)})^{2}}{\sqrt{\sum_{i=m+1}^{m+n} (\alpha z_1^{(i)})^2 + \sum_{i=m+1}^{m+n} (z_2^{(i)})^2}}.
\end{split}
\end{equation}
\endgroup
The correct loss should be proportional to $\alpha\in [0,1]$:
\begin{equation}
\frac{\partial \mathcal{L}_{\text{C$^2$R}}(X)}{\partial \alpha} \geq 0.
\end{equation}
Introducing the shorthand notation:
\begin{equation}
\begin{split}
A&=\sum_{i=1}^{m} (z_{1}^{(i)})^2\\
B&=\sum_{i=m+1}^{m+n} (z_2^{(i)})^2\\
C&=\sum_{i=m+1}^{m+n}(z_{1}^{(i)})^{2}
\end{split}
\end{equation}
The inequality simplifies to:
\begin{equation}
\frac{\alpha C}{\sqrt{\alpha^2 C + B}} \geq \frac{(1-\alpha)C}{
\sqrt{A + (1-\alpha)^2 C}}.
\end{equation}
The final form is:
\begin{equation}
\alpha \geq \frac{1}{\sqrt{\frac{\sum_{i=1}^{m}(z_{1}^{(i)})^{2}}{\sum_{i=m+1}^{m+n}(z_{2}^{(i)})^{2}}}+1}
\end{equation}

\section{Efficient Implementation Details}
\label{appx:efficiency}

Applying the C$^2$R constraint naively requires computing the cosine similarity between all pairs of decoder weight vectors, resulting in a $k \times k$ similarity matrix. For a dictionary size of $k=65,536$, this operation requires $O(k^2)$ memory and computation, which significantly slows down training and increases VRAM usage.

To address this, we implement two engineering optimizations following the methodology of OrtSAE~\citep{korznikov2025ort}:

\paragraph{Chunk-wise Approximation.}
Instead of searching for the nearest neighbor $j^*(i)$ across the entire dictionary, we randomly permute the feature indices at each step and partition the dictionary into smaller blocks (chunks). For a dictionary size $k$ and chunk size $C$, we partition the latents into $N = k/C$ chunks. The nearest neighbor search and loss computation are then restricted to within each chunk.
\begin{equation}
    j^*(i) \approx \operatorname*{arg\,max}_{j \in \text{Chunk}(i), j \neq i} \, \langle \hat{w}_i, \hat{w}_j \rangle.
\end{equation}
In our experiments, we use a chunk size of $C=8,192$. This reduces the complexity from $O(k^2)$ to $O(C \cdot C)$. Since high-cosine similarity features (redundant pairs) are rare and the permutation is randomized at every step, the probability of a redundant pair falling into the same chunk accumulates rapidly over training steps, ensuring the regularization remains effective.

\paragraph{Periodic Updates.}
To further reduce the computational overhead, we compute the C$^2$R loss and its gradients only every $T$ training steps rather than at every iteration. To maintain the same effective regularization strength over time, we scale the regularization coefficient $\lambda_{\text{C$^2$R}}$ by the period $T$:
\begin{equation}
    \mathcal{L}_{\text{step } t}(X) = \mathcal{L}_{\text{SAE}}(X) + (\mathbbm{1}_{t \text{ mod } T = 0} \cdot T \cdot \lambda_{\text{C$^2$R}}) \mathcal{L}_{\text{C$^2$R}}(X).
\end{equation}
We set $T=5$ for all experiments of OrtSAEs and our approach. This configuration aligns our training cost with other SAEs' training times, adding only negligible overhead (slightly more than the findings in OrtSAE, where overhead was reduced to $<10\%$).

\section{Prompts for LLMs and Instructions for Human Annotators}
\label{appx:prompts}

This appendix presents the detailed prompts used in the interpretability evaluation.  
The prompts were designed to elicit consistent reasoning from both LLM-based and human evaluators.  
Two types of LLMs were employed: an \textit{Explainer} to describe latent semantics and a \textit{Predictor} (or Judge) to estimate latent activation likelihood.  
For comparison, human annotators followed analogous instructions.

\subsection{LLM Explainer Prompt}
An example of the prompts used to generate textual descriptions for each latent activation is shown in Table~\ref{tab:llm_explain}.

\begin{table*}[htbp]
\centering
\begin{tabular}{|p{\textwidth}|}
\hline
\textbf{Prompt example for the LLM explainer} \\ \hline\hline
We're studying neurons in a neural network. Each neuron activates on some particular word/words/substring/concept in a short document. The activating words in each document are indicated with <<{{token}}[act:{{activation}}]>>.\\ \\
We will give you a list of ACTIVATE documents, where the neuron fires, ordered by strength. Look at the marked parts of the ACTIVATE documents and summarize in a single sentence what the neuron is activating on. Try not to be overly specific or overly broad. Your explanation should cover most or all activating words. Pay attention to things like capitalization and punctuation if relevant. Keep the explanation as short and simple as possible, limited to 30 words or less. Omit punctuation and formatting. Some examples: "This neuron activates on the word 'knows' in rhetorical questions", and "This neuron activates on verbs related to decision-making and preferences", and "This neuron activates on the substring 'Ent' at the start of words", and "This neuron activates on text about government economic policy".\newline

The relevant documents are given below:\newline

ACTIVATE (1).  see he was enjoying the other shapes too – the<< round[act:55.5]>> bowl and basket and the books underneath them, the
\newline
ACTIVATE (2).  The tube may be<< cylindrical[act:16.75]>> (or conical) with<< circular[act:55.0]>>, rectangular or any desired cross section.↵By
\newline
ACTIVATE (3).  the factors affecting the appearance of impact craters↵The<< circular[act:53.0]>> features so obvious on the Moon's surface are
\newline
ACTIVATE (4).  example, the simple cylindrical case which cylinder has a<< circular[act:52.5]>> cross section, will be considered in detail. If
\newline
ACTIVATE (5).  when Picasaweb closed. They consist of a<< circular[act:50.0]>> emitter Psurrounded by a ring shaped base N
\newline
ACTIVATE (6).  home.↵A mosquito bite appears as an itchy<< round[act:50.0]>>, red, or pink skin bump. It'
\newline
ACTIVATE (7).  of the large square sew-on. Take the<< round[act:48.25]>> sew-on and glue it on the left side
\newline
ACTIVATE (8).  multiple locations to accommodate various connection sizes and elevations.<< Round[act:47.75]>> or rectangular shapes available per design specifications. Knock-
\newline
ACTIVATE (9). . Tie and suspend with gold thread from either our<< round[act:44.25]>> hoop or a stick of your choice.↵Hang
\newline
ACTIVATE (10). bles has been played over the centuries with everything from<< rounded[act:41.75]>> sea pebbles to fruit pits, today the game is
\newline
ACTIVATE (11).  discount on Flashflight.com's most popular<< circular[act:41.25]>> and spherical objects. From now until March 1
\newline
ACTIVATE (12).  front extension, stone corbelling under eaves,<< circular[act:41.25]>> light in gable peak, slender turret with Christian cross
\newline
ACTIVATE (13). piles of rock (called ejecta) around the<< circular[act:41.25]>> hole as well as↵bright streaks of target material
\newline
ACTIVATE (14).  waterproof back and an outer back with 16<< round[act:41.0]>> openings. Manufactured in 1967,
\newline
ACTIVATE (15).  A showcases her gorgeous slender body with swollen breasts,<< round[act:40.75]>> butt, and slender toes on the veranda.↵
\newline
\\
\hline
\end{tabular}
\caption{Prompt example for LLM explainer to explain a latent based on its activations.}
\label{tab:llm_explain}
\end{table*}

\subsection{LLM Predictor Prompt}
An example of the prompts used by the LLM judge (the predictor) to decide whether the described latent would activate for each unseen test sample is shown in Table~\ref{tab:llm_predict}.

\begin{table*}[htbp]
\centering
\begin{tabular}{|p{\textwidth}|}
\hline
\textbf{Prompt example for the LLM predictor} \\ \hline\hline
We're studying neurons in a neural network. Each neuron activates on some particular word/words/substring/concept in a short document. You will be given a short explanation of what this neuron activates for, and then be shown 15 example sequences in random order. You will have to return a comma-separated list of the examples where you think the neuron should activate at least once, on ANY of the words or substrings in the document. For example, your response might look like "1, 2, 6, 9, 12". Try not to be overly specific in your interpretation of the explanation. If you think there are no examples where the neuron will activate, you should just respond with "None". You should include nothing else in your response other than comma-separated numbers or the word "None" - this is important.\newline

Here is the explanation: this neuron fires on words describing round or circular shapes including round circular rounded and cylindrical.\newline

Here are the examples:\newline

1. in South Africa • Uganda) · Asia (in China • India • Myanmar • Pakistan • Taiwan • Japan\newline
2. ized was either beheaded or shot at point blank range." more >>↵A Syrian mother and widow was tortured\newline
3.  method' anyone can do that but getting the right mindset to succeed. This is something most traders simply cannot\newline
4.  Facebook Be Fixed?↵CMS Wire (May 24, 2012) - Facebook\newline
5. . Finished in a weathered brown and accented with a circular polished silver bezel. The metal dial has polished silver\newline
6.  when a customer has changed his or her mind about a transaction, or when an error has occurred, the\newline
7. . Several of you have reached out to us and to our colleagues across the Administration. You've warned\newline
8.  crater is that↵you cannot see it. Its circular structure is nearly a kilometer below the↵surface and\newline
9.  sound crazy? Okay?↵DW:I’m just going to tell you the truth.↵THE\newline
10.  am I missing some key information here?<eos>The eleventh round of 2020 Monster Energy Super\newline
11.  awesome Flashflight Light-Up Flying Discs are circular, and our equally saucy Meteorlight LED Light\newline
12.  rather than doubling Defense on Dodge↵Strength ••, Brawl •↵Add Brawl rather than doubling Defense on Dodge\newline
13.  achieve a perfectly snug fit. Lastly the grain is circular-grained, after which the stone will not move\newline
14. . Instead he drew a Dalek with two big round holes in it, and a guy catching a baseball\newline
15. abb, Sean McDermott, Kevin Kolb). If the Eagles are waiting for a Packers assistant, the best\newline
\\
\hline
\end{tabular}
\caption{Prompt example for LLM judger to predict latent activations based on its explanation generated by the LLM explainer.}
\label{tab:llm_predict}
\end{table*}

\subsection{Human Annotator Instruction}
Human annotators were provided with the latent's same activating samples and unseen samples as the LLM explainer and the LLM predictor. An example of the instructions is shown in Table~\ref{tab:human_instructions}.

\begin{table*}[htbp]
\centering
\begin{tabular}{|p{\textwidth}|}
\hline
\textbf{Instruction example for human annotators} \\ \hline\hline
~\newline
We're studying neurons in a neural network. Each neuron activates on some particular word/words/substring/concept in a short document. The activating words in each document are highlighted.
\newline
We will give you a list of ACTIVATE documents (where the neuron fires, ordered by strength), please look at the marked parts of the ACTIVATE documents. Summarize in a single sentence what the neuron is activating on. Try not to be overly specific or overly broad. Your explanation should cover all activating words. Pay attention to things like capitalization and punctuation if relevant. Keep the explanation as short and simple as possible, limited to 30 words or less. Omit punctuation and formatting. 
\newline Some examples: "This neuron activates on the word 'knows' in rhetorical questions", and "This neuron activates on verbs related to decision-making and preferences", and "This neuron activates on the substring 'Ent' at the start of words", and "This neuron activates on text about government economic policy".

The relevant documents are given below:\newline

ACTIVATE (1). TR\colorbox[rgb]{0.988,0.988,0.996}{{\strut}UDA} \colorbox[rgb]{0.522,0.612,0.980}{{\strut}passes} \colorbox[rgb]{0.718,0.773,0.988}{{\strut}into} \colorbox[rgb]{0.710,0.765,0.988}{{\strut}your} \colorbox[rgb]{0.902,0.922,0.996}{{\strut}breast} \colorbox[rgb]{0.718,0.773,0.988}{{\strut}milk}\colorbox[rgb]{0.608,0.682,0.984}{{\strut}.} \colorbox[rgb]{0.859,0.886,0.992}{{\strut}Continue} \colorbox[rgb]{0.522,0.612,0.980}{{\strut}to} \colorbox[rgb]{0.302,0.435,0.973}{{\strut}take} prednis\colorbox[rgb]{0.984,0.988,0.996}{{\strut}olone} \colorbox[rgb]{0.435,0.541,0.976}{{\strut}regularly} \colorbox[rgb]{0.655,0.722,0.984}{{\strut}until} \colorbox[rgb]{0.765,0.808,0.988}{{\strut}your} \colorbox[rgb]{0.737,0.788,0.988}{{\strut}doctor} \colorbox[rgb]{0.576,0.655,0.980}{{\strut}tells} \colorbox[rgb]{0.561,0.647,0.980}{{\strut}you} \colorbox[rgb]{0.690,0.749,0.984}{{\strut}to}\newline ACTIVATE (2). 0.4\% vs. You \colorbox[rgb]{0.671,0.733,0.984}{{\strut}may} \colorbox[rgb]{0.588,0.667,0.980}{{\strut}take} \colorbox[rgb]{0.490,0.584,0.976}{{\strut}this} \colorbox[rgb]{0.353,0.478,0.973}{{\strut}medicine} \colorbox[rgb]{0.549,0.635,0.980}{{\strut}with} \colorbox[rgb]{0.831,0.863,0.992}{{\strut}or} \colorbox[rgb]{0.808,0.843,0.992}{{\strut}without} \colorbox[rgb]{0.878,0.902,0.992}{{\strut}meals}\colorbox[rgb]{0.820,0.855,0.992}{{\strut}.} Please once you are cured\newline ACTIVATE (3). ↵\colorbox[rgb]{0.949,0.957,0.996}{{\strut}How} the \colorbox[rgb]{0.690,0.749,0.984}{{\strut}interaction} \colorbox[rgb]{0.831,0.863,0.992}{{\strut}occurs}\colorbox[rgb]{0.875,0.898,0.992}{{\strut}:}↵\colorbox[rgb]{0.863,0.890,0.992}{{\strut}When} \colorbox[rgb]{0.698,0.753,0.984}{{\strut}these} \colorbox[rgb]{0.890,0.914,0.992}{{\strut}two} \colorbox[rgb]{0.400,0.514,0.976}{{\strut}medicines} \colorbox[rgb]{0.608,0.682,0.984}{{\strut}are} \colorbox[rgb]{0.682,0.745,0.984}{{\strut}taken} \colorbox[rgb]{0.467,0.569,0.976}{{\strut}together}\colorbox[rgb]{0.682,0.745,0.984}{{\strut},} cime\colorbox[rgb]{0.980,0.984,0.996}{{\strut}tidine} \colorbox[rgb]{0.541,0.627,0.980}{{\strut}may} \colorbox[rgb]{0.584,0.663,0.980}{{\strut}cause} \colorbox[rgb]{0.584,0.663,0.980}{{\strut}your} \colorbox[rgb]{0.765,0.808,0.988}{{\strut}body}\newline ACTIVATE (4). c.com.$<$eos$>$Serious. These \colorbox[rgb]{0.976,0.980,0.996}{{\strut}medicines} \colorbox[rgb]{0.584,0.663,0.980}{{\strut}may} \colorbox[rgb]{0.408,0.522,0.976}{{\strut}interact} \colorbox[rgb]{0.616,0.690,0.984}{{\strut}and} \colorbox[rgb]{0.718,0.773,0.988}{{\strut}cause} \colorbox[rgb]{0.757,0.804,0.988}{{\strut}very} \colorbox[rgb]{0.765,0.808,0.988}{{\strut}harmful} \colorbox[rgb]{0.757,0.804,0.988}{{\strut}effects}\colorbox[rgb]{0.937,0.949,0.996}{{\strut}.} \colorbox[rgb]{0.710,0.765,0.988}{{\strut}Contact} \colorbox[rgb]{0.710,0.765,0.988}{{\strut}your} \colorbox[rgb]{0.835,0.867,0.992}{{\strut}healthcare} \colorbox[rgb]{0.816,0.851,0.992}{{\strut}professional}
\newline
...
\newline
ACTIVATE (15).  bruising, \colorbox[rgb]{1.000,1.000,1.000}{{\strut}or} dark stools\colorbox[rgb]{0.859,0.886,0.992}{{\strut},} \colorbox[rgb]{0.784,0.824,0.988}{{\strut}contact} \colorbox[rgb]{0.855,0.882,0.992}{{\strut}your} \colorbox[rgb]{0.831,0.863,0.992}{{\strut}doctor} \colorbox[rgb]{0.808,0.847,0.992}{{\strut}right} \colorbox[rgb]{0.671,0.733,0.984}{{\strut}away}\colorbox[rgb]{0.725,0.776,0.988}{{\strut}.}\colorbox[rgb]{0.682,0.745,0.984}{{\strut}Your} \colorbox[rgb]{0.863,0.890,0.992}{{\strut}healthcare} \colorbox[rgb]{0.765,0.808,0.988}{{\strut}professionals} \colorbox[rgb]{0.725,0.776,0.988}{{\strut}may} already be \colorbox[rgb]{0.906,0.922,0.996}{{\strut}aware} \colorbox[rgb]{0.875,0.898,0.992}{{\strut}of} \colorbox[rgb]{0.847,0.875,0.992}{{\strut}this}
\newline

Based on your explanation of what this neuron activates for, please review the following 15 examples and indicate if you believe the neuron should activate at least once on ANY of the words or substrings within the document. Provide the corresponding text IDs. For instance, your response might look like "2, 3, 5, 6, 13". Avoid being overly specific in your interpretation of the explanation.

Here are the examples:
\newline

1.  off-the-wall in this first directorial effort from the 49-year-old Belgian\newline 2.  your organization.$<$eos$>$The Academy of Motion Picture Arts and Sciences which is best known for organizing the Oscars has\newline 3.  Realtors and the Mortgage Bankers Association.↵But this time, lobbyists are worried. That's because\newline 4.  considered a natural antihistamine. Valtrex is used to treat herpes zoster and herpes simplex and,
\newline

...
\newline

15.  of the stomach and intestines.↵Be sure to tell your doctor if you experience any of these side effects \newline
\\
\hline
\end{tabular}
\caption{Instruction example for human annotators to predict latent activation.}
\label{tab:human_instructions}
\end{table*}


\section{User Study Details}
\label{appx:user_study}

\begin{table}[t]
\renewcommand{\arraystretch}{1.1}
\centering
\begin{tabular}{|l|c|c|}
\hline
 & Prediction Match & Pearson $r$\\
\hline
Annotator 1 & 97.8\% & 0.74  \\
Annotator 2 & 96.7\% & 0.63  \\
Annotator 3 & 97.6\% & 0.86  \\
\hline
Average & 97.3\% & 0.74  \\
\hline
\end{tabular}
\caption{User study comparing GPT-5-mini~\citep{gpt5mini2024} and human annotators in the automated interpretability task. The table reports the match rate and Pearson correlation~\citep{pearson1895vii} between human- and LLM-derived AutoInterp scores.}
\label{tab:user_study}
\end{table}

We recruited three human annotators with high-school-level English proficiency, who replicated the explainer–judge process on 30 latents sampled from Batch TopK SAE and its C$^2$R-enhanced variant at layer 12 of \textit{Gemma-2-2B}, covering five L$_0$ settings. Each annotator produced 450 activation predictions. 
The probability of agreement between human annotators and GPT-5-mini, as well as the Pearson correlation coefficient~\citep{pearson1895vii} between human- and LLM-derived scores, are summarized in Table~\ref{tab:user_study}. 
GPT-5-mini achieved a \textbf{97.3\%} match rate with humans and a Pearson $r$ of \textbf{0.74}, confirming its suitability as a proxy evaluator for automated interpretability assessment.

The annotators are recruited from the university, and the compensation was set according to the standard payment guidelines for on-campus research participation.

To verify that our results are not sensitive to the specific judge model, we re-evaluate the AutoInterp scores of all methods using GPT-5 and compute the per-method Pearson correlation with the original GPT-5-mini scores. As shown in Table~\ref{tab:judge_consistency}, all correlations exceed 0.81 with $p < 0.001$, confirming high consistency.

\begin{table}[t]
\renewcommand{\arraystretch}{1.1}
\centering
\begin{tabular}{|l|c|c|c|}
\hline
Method & GPT-5-mini & GPT-5 & Pearson $r$ \\
\hline
C$^2$R     & 0.9239 & 0.9218 & 0.94 \\
Batch TopK & 0.9208 & 0.9161 & 0.89 \\
Matryoshka & 0.9120 & 0.9031 & 0.89 \\
Ort        & 0.9196 & 0.9153 & 0.81 \\
TopK       & 0.8762 & 0.8735 & 0.94 \\
\hline
\end{tabular}
\caption{Consistency between GPT-5-mini and GPT-5 as AutoInterp judges. All Pearson correlations are significant ($p < 0.001$).}
\label{tab:judge_consistency}
\end{table}

\section{Detailed Metrics Description}
\label{appx:metrics_desc}
The following six key metrics are used to evaluate the performance of SAEs integrated with C$^2$R:

\begin{itemize}
    \item \textbf{Loss Recovered}
    This metric is the primary measure of \textbf{reconstruction fidelity}. It quantifies the degree to which an SAE can preserve the original language model's Next-Token Prediction performance after its internal activations are reconstructed. It is defined as:
    $
    \text{Loss Recovered} = \frac{(H^{*}-H_{0})}{(H_{orig}-H_{0})}
    $
    where $H_{orig}$ is the original cross-entropy loss, $H^{*}$ is the loss after replacement with SAE-reconstructed activations, and $H_{0}$ is the loss after zero-ablating the original activations. A higher value indicates better reconstruction fidelity.
    
    \item \textbf{KL Div. Score~\citep{gaoscaling}}
    As a complementary measure of reconstruction quality, this metric assesses how effectively the SAE's reconstruction recovers the model's output distribution from a zero-ablated baseline. It is a normalized score that quantifies the reduction in Kullback-Leibler (KL) divergence between the output logits and the original model's logits distribution ($P_{orig}$). The score is calculated as:
    $
    \text{KL Div. Score} = 1 - \frac{D_{KL}(P_{SAE} \parallel P_{orig})}{D_{KL}(P_{ablated} \parallel P_{orig})}
    $
    where $D_{KL}(P_{ablated} \parallel P_{orig})$ is the KL divergence when the activation is zero-ablated, and $D_{KL}(P_{SAE} \parallel P_{orig})$ is the KL divergence when the activation is replaced by the SAE reconstruction. This score is bounded between $0$ and $1$, where a higher score signifies superior reconstruction performance relative to the zero-ablated state.
    
    \item \textbf{AutoInterp~\citep{pauloautomatically}}
    This metric evaluates the human-understandability of learned latents using LLMs. It operates in two stages: an LLM generates a feature description based on activating inputs, and another LLM judge uses this description to predict latent activation on new sequences. The prediction accuracy serves as the AutoInterp score.

    \item \textbf{Split Num~\citep{chanin2024absorption}}
    A diagnostic metric focusing on feature splitting. This metric serves as a proxy for the granularity and non-redundancy of the learned latents. It is measured by identifying a single high-level concept (e.g., all tokens starting with a specific letter) and counting the minimum number of distinct SAE latents required to significantly improve classification performance on a probe for that concept. A higher count can indicate that more latent are fragmented into distinct components.
    
    \item \textbf{Absorption~\citep{chanin2024absorption}}
    This metric focuses on feature absorption. It measures the tendency of an SAE to learn two coupled hierarchical features (e.g., A and ``B excluding A'') instead of two independent features (A and B). The metric is calculated by diagnosing the activation patterns of the general SAE latents: measuring the frequency at which they fail to activate when a token-aligned child latent is present in the input. A lower score is desirable, indicating better feature isolation and reduced feature absorption.

    \item \textbf{Max CosSim}
    This metric quantifies the maximum cosine similarity between the decoder weight vectors of all learned latents, reflecting the feature composition level of SAEs~\citep{bussmannlearning}. High similarity suggests significant directional overlap or redundancy among SAE latents.

    \item \textbf{Disentanglement~\citep{huang-etal-2024-ravel}} This benchmark evaluates the ability of interpretability methods to disentangle independent attributes within language model representations. It utilizes interchange interventions to test whether a targeted feature (e.g., city country) can be modified without affecting other related attributes (e.g., city language). The performance is summarized by the disentangle score: $ \text{Disentangle Score} = \frac{1}{2} (\text{Cause} + \text{Iso}) $ where \textbf{Cause} measures the success rate of changing the target attribute's value through intervention, and \textbf{Iso} (Isolation) measures the frequency with which non-target attributes remain unchanged. A high score indicates that the SAE has successfully localized individual concepts into independent, causal units. 
    
\end{itemize}

\section{GPU Budget}
We ran all SAE training experiments utilizing an NVIDIA H800 GPU, consuming a total of 300 GPU hours and achieving a peak memory utilization of 65 GB.

Evaluating a trained SAE across all metrics required approximately 1 GPU hour. The breakdown of evaluation time per SAE is as follows:
\begin{itemize}
    \item Reconstruction Fidelity Metrics: ~30 minutes
    \item Interpretability Analysis: ~10 minutes
    \item Feature Hierarchy Metrics: ~10 minutes
    \item Disentanglement Metrics: ~15 minutes
\end{itemize}
In terms of memory footprint, the reconstruction fidelity evaluation was the most demanding, requiring up to 50 GB of VRAM. The interpretability and feature hierarchy analyses were less memory-intensive, each requiring approximately 15 GB.

\section{Robustness and Generalization}\label{app:robustness}

\begin{table}[h]
\centering
\begin{tabular}{lccccc}
\toprule
Method & KL Score $\uparrow$ & Interp $\uparrow$ & Composition $\downarrow$ & Absorption $\downarrow$ & Split $\downarrow$ \\
\midrule
Batch TopK & $\underline{0.9621 \pm 0.0006}$ & $\mathbf{0.9311 \pm 0.0098}$ & $0.3325 \pm 0.0003$ & $0.1908 \pm 0.0240$ & $1.0962 \pm 0.0722$ \\
Matryoshka & $0.9597 \pm 0.0008$ & $\mathbf{0.9311 \pm 0.0120}$ & $0.1870 \pm 0.0003$ & $\mathbf{0.0360 \pm 0.0066}$ & $1.0673 \pm 0.0361$ \\
Ort        & $0.9615 \pm 0.0006$ & $0.9297 \pm 0.0082$ & $\underline{0.1038 \pm 0.0004}$ & $0.0819 \pm 0.0203$ & $\underline{1.0577 \pm 0.0218}$ \\
C$^2$R     & $\mathbf{0.9628 \pm 0.0004}$ & $\underline{0.9309 \pm 0.0052}$ & $\mathbf{0.0991 \pm 0.0003}$ & $\underline{0.0656 \pm 0.0185}$ & $\mathbf{1.0385 \pm 0.0238}$ \\
\bottomrule
\end{tabular}
\caption{Statistical validation with 95\% confidence intervals over 5 independent runs (Gemma-2-2B, layer 12, $k{=}100$). \textbf{Bold} and \underline{underline} indicate the best and second-best values, respectively.}
\label{tab:stat_valid}
\end{table}

\begin{table}[h]
\centering
\begin{tabular}{lccccc}
\toprule
Method & KL Score $\uparrow$ & Interp $\uparrow$ & Composition $\downarrow$ & Absorption $\downarrow$ & Split $\downarrow$ \\
\midrule
Batch TopK & $\mathbf{0.9889}$ & $0.9220$ & $0.2677$ & $0.0074$ & $1.1154$ \\
Matryoshka & $0.9870$ & $\underline{0.9277}$ & $0.1422$ & $\underline{0.0032}$ & $1.1154$ \\
Ort        & $\underline{0.9886}$ & $0.9214$ & $\underline{0.0945}$ & $0.0034$ & $1.1154$ \\
C$^2$R     & $0.9874$ & $\mathbf{0.9333}$ & $\mathbf{0.0675}$ & $\mathbf{0.0018}$ & $1.1154$ \\
\bottomrule
\end{tabular}
\caption{Results on Qwen3-8B (layer 20, $k{=}100$).}
\label{tab:qwen3}
\end{table}

\begin{table}[h]
\centering
\begin{tabular}{lccccc}
\toprule
Method & KL Score $\uparrow$ & Interp $\uparrow$ & Composition $\downarrow$ & Absorption $\downarrow$ & Split $\downarrow$ \\
\midrule
Batch TopK & $\mathbf{0.9860}$ & $0.9405$ & $0.3334$ & $0.1470$ & $1.1923$ \\
Matryoshka & $0.9853$ & $\mathbf{0.9552}$ & $0.1803$ & $\mathbf{0.0349}$ & $\underline{1.1154}$ \\
Ort        & $\underline{0.9859}$ & $0.9265$ & $\underline{0.1200}$ & $0.0979$ & $\mathbf{1.0385}$ \\
C$^2$R     & $0.9858$ & $\underline{0.9443}$ & $\mathbf{0.0816}$ & $\underline{0.0410}$ & $\mathbf{1.0385}$ \\
\bottomrule
\end{tabular}
\caption{Results on Llama-3-8B (layer 20, $k{=}100$).}
\label{tab:llama3}
\end{table}

\begin{table}[h]
\centering
\begin{tabular}{lccccc}
\toprule
Method & KL Score $\uparrow$ & Interp $\uparrow$ & Composition $\downarrow$ & Absorption $\downarrow$ & Split $\downarrow$ \\
\midrule
TopK       & $0.9720$ & $0.9307$ & $0.3270$ & $0.1325$ & $1.5769$ \\
Batch TopK & $0.9706$ & $0.9323$ & $0.3329$ & $0.1087$ & $1.3462$ \\
Matryoshka & $\mathbf{0.9727}$ & $0.9410$ & $0.1781$ & $\mathbf{0.0205}$ & $\mathbf{1.2308}$ \\
Ort        & $0.9722$ & $\underline{0.9543}$ & $\underline{0.1000}$ & $0.0352$ & $\underline{1.2692}$ \\
C$^2$R     & $\mathbf{0.9727}$ & $\mathbf{0.9570}$ & $\mathbf{0.0781}$ & $\underline{0.0260}$ & $\underline{1.2692}$ \\
\bottomrule
\end{tabular}
\caption{Results on Gemma-2-2B layer 20 ($k{=}100$).}
\label{tab:gemma_l20}
\end{table}

\begin{table}[h]
\centering
\begin{tabular}{lccccc}
\toprule
Method & KL Score $\uparrow$ & Interp $\uparrow$ & Composition $\downarrow$ & Absorption $\downarrow$ & Split $\downarrow$ \\
\midrule
TopK       & $\mathbf{0.9631}$ & $0.9196$ & $0.3416$ & $0.2987$ & $1.2692$ \\
Batch TopK & $0.9621$ & $\mathbf{0.9313}$ & $0.3430$ & $0.2149$ & $1.1538$ \\
Matryoshka & $0.9598$ & $0.8870$ & $0.1892$ & $\mathbf{0.0241}$ & $\underline{1.0769}$ \\
Ort        & $\underline{0.9629}$ & $0.9072$ & $\underline{0.1053}$ & $0.1058$ & $\mathbf{1.0385}$ \\
C$^2$R     & $\mathbf{0.9631}$ & $\underline{0.9302}$ & $\mathbf{0.0870}$ & $\underline{0.0654}$ & $\mathbf{1.0385}$ \\
\bottomrule
\end{tabular}
\caption{Results with 1B training tokens on OpenWebText (Gemma-2-2B, layer 12, $k{=}100$).}
\label{tab:1b_owt}
\end{table}

\begin{table}[h]
\centering
\begin{tabular}{lccccc}
\toprule
Method & KL Score $\uparrow$ & Interp $\uparrow$ & Composition $\downarrow$ & Absorption $\downarrow$ & Split $\downarrow$ \\
\midrule
Batch TopK & $0.9606$ & $\underline{0.9211}$ & $0.3227$ & $0.1920$ & $1.1538$ \\
Matryoshka & $0.9602$ & $0.9117$ & $0.1781$ & $\mathbf{0.0158}$ & $\mathbf{1.0385}$ \\
Ort        & $\underline{0.9611}$ & $0.9202$ & $\underline{0.1053}$ & $0.0627$ & $\underline{1.0769}$ \\
C$^2$R     & $\mathbf{0.9618}$ & $\mathbf{0.9424}$ & $\mathbf{0.0766}$ & $\underline{0.0396}$ & $\mathbf{1.0385}$ \\
\bottomrule
\end{tabular}
\caption{Results on The Pile (Gemma-2-2B, layer 12, $k{=}100$).}
\label{tab:pile}
\end{table}

\section{Ablation Study}\label{app:ablation}
Table~\ref{tab:ablation} reports the component ablation results, where removing the nearest-neighbor restriction (NoNNR) or the ReLU cosine gate (NoRCG) leads to severe degradation in reconstruction fidelity or inflated composition scores. Table~\ref{tab:lambda} reports the sensitivity of C$^2$R to the regularization strength $\lambda_{\text{C}^2\text{R}}$, showing that $\lambda_{\text{C}^2\text{R}}{=}5$ achieves the best trade-off.

\begin{table}[h]
\centering
\begin{tabular}{lccccc}
\toprule
Method & KL Score $\uparrow$ & Interp $\uparrow$ & Composition $\downarrow$ & Absorption $\downarrow$ & Split $\downarrow$ \\
\midrule
Batch TopK       & $0.9598$ & $0.9208$ & $0.3321$ & $0.1985$ & $\underline{1.1154}$ \\
Ort              & $\underline{0.9617}$ & $0.9208$ & $\underline{0.1046}$ & $0.0606$ & $\underline{1.1154}$ \\
C$^2$R (NoNNR)   & $0.9439$ & $\underline{0.9488}$ & $0.5655$ & $\mathbf{0.0255}$ & $1.1538$ \\
C$^2$R (NoRCG)   & $0.8439$ & $\mathbf{0.9873}$ & $0.4326$ & $\underline{0.0520}$ & $1.1923$ \\
C$^2$R           & $\mathbf{0.9629}$ & $0.9239$ & $\mathbf{0.0990}$ & $0.0590$ & $\mathbf{1.0769}$ \\
\bottomrule
\end{tabular}
\caption{Component ablation. NoNNR: without the nearest-neighbor restriction. NoRCG: without the ReLU cosine gate. Removing either component causes a severe degradation in reconstruction fidelity or feature composition.}
\label{tab:ablation}
\end{table}

\begin{table}[h]
\centering
\begin{tabular}{lccccc}
\toprule
$\lambda_{\text{C}^2\text{R}}$ & KL Score $\uparrow$ & Interp $\uparrow$ & Composition $\downarrow$ & Absorption $\downarrow$ & Split $\downarrow$ \\
\midrule
-- (Batch TopK) & $0.9598$ & $0.9196$ & $0.3321$ & $0.1985$ & $1.1154$ \\
-- (Ort)        & $0.9617$ & $0.9196$ & $0.1046$ & $0.0606$ & $1.1154$ \\
\midrule
$0.1$  & $0.9629$ & $0.9291$ & $0.2944$ & $0.1926$ & $\mathbf{1.0385}$ \\
$0.5$  & $\mathbf{0.9638}$ & $\mathbf{0.9496}$ & $0.2138$ & $0.1726$ & $\mathbf{1.0385}$ \\
$1$    & $\underline{0.9633}$ & $\underline{0.9407}$ & $0.1676$ & $0.1349$ & $1.1154$ \\
$5$    & $0.9629$ & $0.9239$ & $\underline{0.0990}$ & $\underline{0.0590}$ & $\underline{1.0769}$ \\
$10$   & $0.9521$ & $0.9168$ & $\mathbf{0.0777}$ & $\mathbf{0.0425}$ & $\mathbf{1.0385}$ \\
\bottomrule
\end{tabular}
\caption{Sensitivity to $\lambda_{\text{C}^2\text{R}}$. Baselines (Batch TopK and Ort) are shown above the rule for reference. $\lambda_{\text{C}^2\text{R}}{=}5$ achieves the best trade-off between reconstruction fidelity and structural metrics.}
\label{tab:lambda}
\end{table}

\section{Compatibility with Other SAE Architectures}\label{app:compatibility}
Table~\ref{tab:backbone} reports the full results of applying C$^2$R to three different SAE backbones: TopK, OrtSAE, and AbsTopK. In each pair, adding C$^2$R improves all structural metrics without degrading reconstruction fidelity, confirming that C$^2$R is a backbone-agnostic regularizer.

\begin{table}[h]
\centering
\begin{tabular}{lccccc}
\toprule
Method & KL Score $\uparrow$ & Interp $\uparrow$ & Composition $\downarrow$ & Absorption $\downarrow$ & Split $\downarrow$ \\
\midrule
TopK             & $0.9620$ & $0.8762$ & $0.3311$ & $0.2680$ & $1.3462$ \\
TopK + C$^2$R    & $\mathbf{0.9627}$ & $\mathbf{0.9227}$ & $\mathbf{0.1004}$ & $\mathbf{0.0752}$ & $\mathbf{1.0769}$ \\
\midrule\midrule
Ort              & $0.9617$ & $0.9196$ & $0.1046$ & $0.0606$ & $1.1154$ \\
Ort + C$^2$R     & $\mathbf{0.9623}$ & $\mathbf{0.9217}$ & $\mathbf{0.0662}$ & $\mathbf{0.0524}$ & $\mathbf{1.0385}$ \\
\midrule\midrule
AbsTopK          & $0.9576$ & $0.6759$ & $0.2357$ & $0.1407$ & $1.3462$ \\
AbsTopK + C$^2$R & $\mathbf{0.9580}$ & $\mathbf{0.7579}$ & $\mathbf{0.0810}$ & $\mathbf{0.0846}$ & $\mathbf{1.0769}$ \\
\bottomrule
\end{tabular}
\caption{C$^2$R applied to different SAE backbones (Gemma-2-2B, layer 12, $k{=}100$). Within each pair, the better value is in \textbf{bold}. C$^2$R consistently improves structural metrics across all architectures.}
\label{tab:backbone}
\end{table}

\section{Downstream Causal Intervention Tasks}\label{app:downstream}

Table~\ref{tab:downstream} reports the results on two causal intervention tasks from SAEBench. C$^2$R achieves the second-best performance on both SCR and TPP, outperforming TopK, Batch TopK, and Ort, and achieving comparable results to Matryoshka.

\begin{table}[h]
\centering
\begin{tabular}{lcc}
\toprule
Method & SCR $\uparrow$ & TPP $\uparrow$ \\
\midrule
TopK       & $0.0481$ & $0.0038$ \\
Batch TopK & $0.0353$ & $0.0045$ \\
Matryoshka & $\mathbf{0.1172}$ & $\mathbf{0.0370}$ \\
Ort        & $0.0859$ & $0.0061$ \\
C$^2$R     & $\underline{0.1047}$ & $\underline{0.0226}$ \\
\bottomrule
\end{tabular}
\caption{Downstream causal intervention tasks (Gemma-2-2B, layer 12, $k{=}100$). SCR: Spurious Correlation Removal. TPP: Targeted Probe Perturbation. \textbf{Bold} and \underline{underline} indicate the best and second-best values, respectively.}
\label{tab:downstream}
\end{table}

\section{Empirical Verification of Eq.~\ref{eq:feature_disp_condition}}\label{app:eq13_verification}

We verify the condition in Eq.~\ref{eq:feature_disp_condition} using our trained SAEs on Gemma-2-2B layer 12 with a 4M-token test set from SAEBench.

Figure~\ref{fig:eq13_ratio} shows the distribution of the log A/B ratio, i.e., $\log \frac{\sum_{i=1}^{m} (z_{1}^{(i)})^{2}}{\sum_{i=m+1}^{m+n} (z_{2}^{(i)})^{2}}$, across all absorption pairs. The median A/B ratio is 81.56, confirming that the cumulative energy of the primary feature strongly dominates the residual in practice.

\begin{figure}[h]
\centering
\includegraphics[width=\columnwidth]{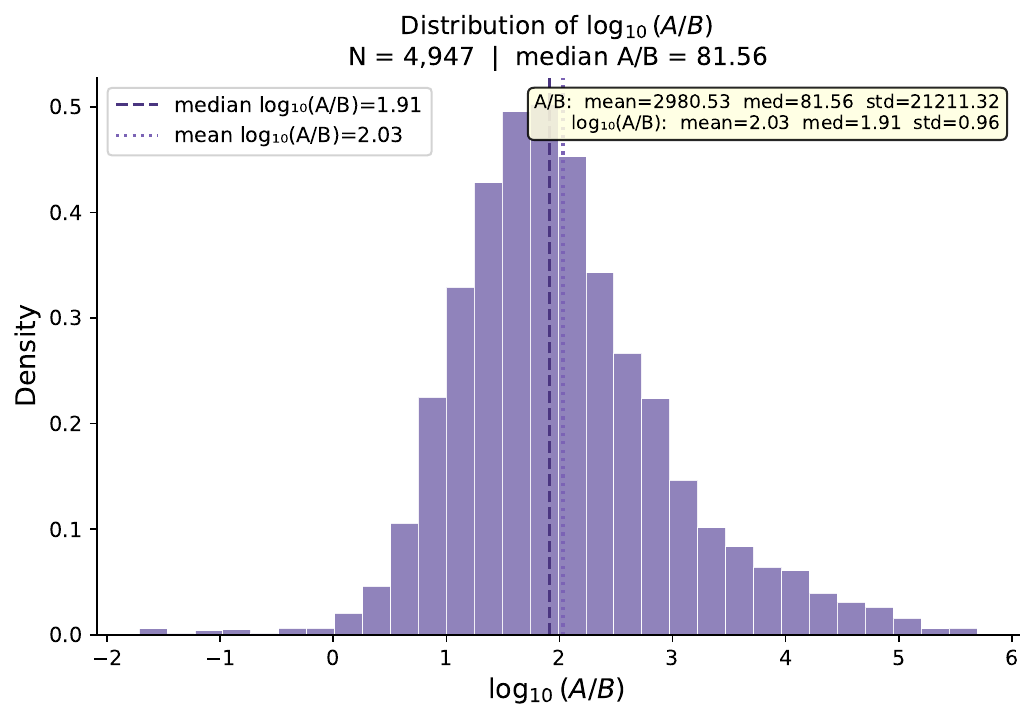}
\caption{Distribution of the log A/B ratio across absorption pairs ($N{=}4{,}947$, median $= 81.56$).}
\label{fig:eq13_ratio}
\end{figure}

Figure~\ref{fig:eq13_scatter} shows the scatter plot of $\alpha$ versus the right-hand side of Eq.~\ref{eq:feature_disp_condition} for all absorption pairs. Points above the diagonal satisfy the condition. 88.1\% of pairs ($N{=}4{,}555$) fall in the satisfied region, and the violated pairs are concentrated in a low-$\alpha$ regime where absorption is minimal.

\begin{figure}[h]
\centering
\includegraphics[width=\columnwidth]{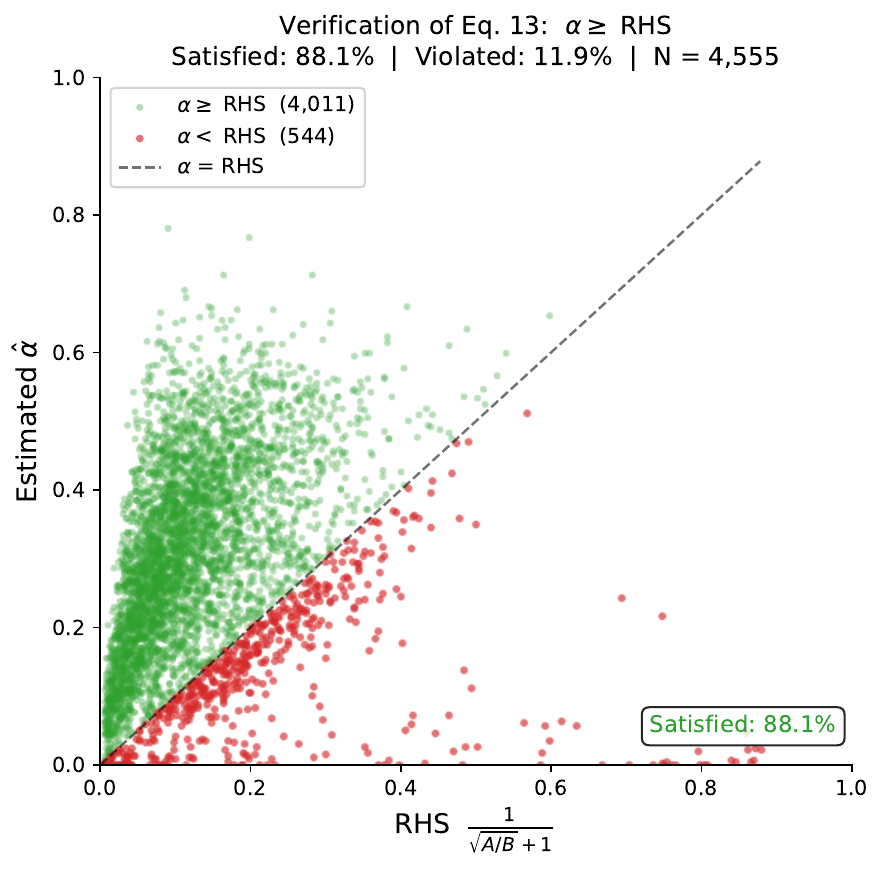}
\caption{Scatter plot of $\alpha$ vs.\ RHS of Eq.~\ref{eq:feature_disp_condition} ($N{=}4{,}555$). 88.1\% of pairs satisfy the condition.}
\label{fig:eq13_scatter}
\end{figure}

\section{Relationship to OrtSAE}\label{appx:ortsae_relation}

The chunk-wise nearest-neighbor calculation in Appendix~\ref{appx:efficiency} is an engineering optimization to reduce complexity and control variables, not a methodological dependence. Fundamentally, OrtSAE constrains decoder weights, whereas C$^2$R operates in activation space using the Minkowski inequality. In C$^2$R, decoder cosine similarity acts as a gating weight ($\rho^2$) to target redundant features, creating a gradient dynamically scaled by activation magnitude $S$. This adaptive approach distinguishes C$^2$R from OrtSAE's uniform penalty, protecting low-frequency features from premature destruction. As shown in Table~\ref{tab:backbone}, C$^2$R can be applied on top of OrtSAE to achieve further improvements, confirming that the two methods are complementary.

\section{Sensitivity to Feature Frequency}\label{appx:rare_features}

Since C$^2$R aggregates batchwise $\ell_2$ norms, the effective regularization strength on a given latent pair scales with how frequently those features appear in the batch. To investigate whether rare features receive insufficient regularization, we analyze the relationship between feature frequency and absorption rate using 26 letter-specific features from SAEBench on a Gemma-2-2B Batch TopK SAE.

As shown in Figure~\ref{fig:absorption_vs_frequency}, the Pearson correlation between feature frequency and absorption rate is $r = +0.301$ ($p = 0.135$), which is not statistically significant. The Spearman correlation is $\rho = +0.403$ ($p = 0.041$), suggesting a slight positive trend where absorption may marginally increase with frequency rather than decrease. Rare features such as Q, J, and X exhibit absorption rates below 0.1, comparable to or lower than high-frequency features such as C and S. This indicates that the consistency pressure remains effective across the frequency spectrum and that rare features do not suffer from reduced reliability due to batch statistics.

\begin{figure}[h]
\centering
\includegraphics[width=0.7\columnwidth]{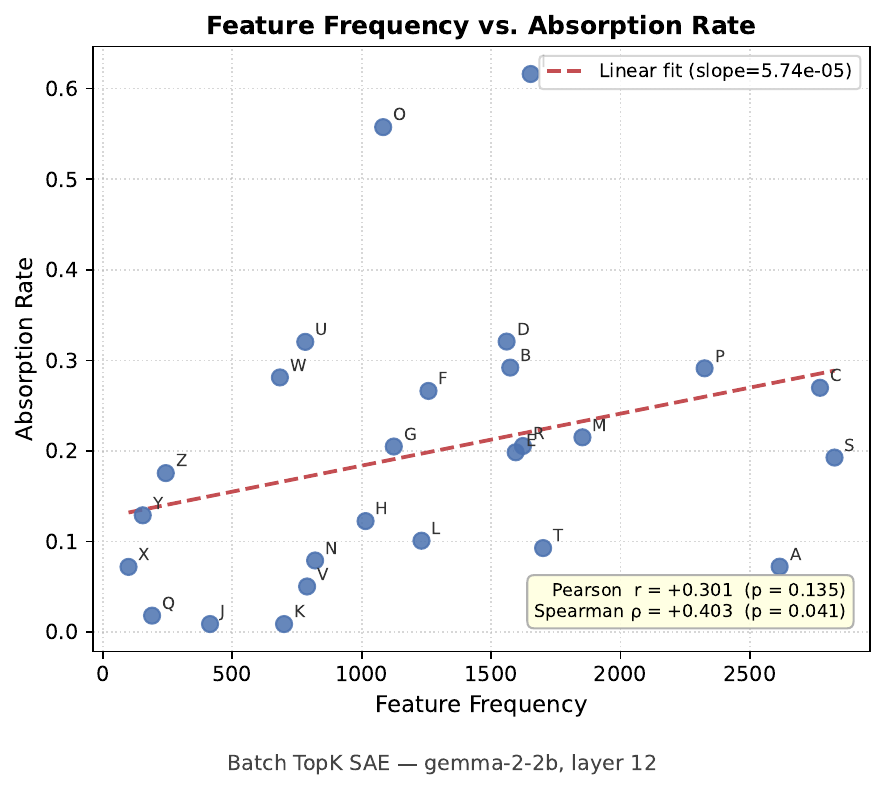}
\caption{Absorption rate vs.\ feature frequency for 26 letter-specific features (Gemma-2-2B, Batch TopK SAE). No significant negative correlation is observed, indicating that rare features do not suffer from higher absorption.}
\label{fig:absorption_vs_frequency}
\end{figure}

\section{Dead Feature Rates}\label{appx:dead_features}

Table~\ref{tab:dead_features} reports the dead feature rates for all methods on Gemma-2-2B layer 12. C$^2$R maintains high dictionary utilization with a dead feature rate under 1\%.

\begin{table}[h]
\centering
\begin{tabular}{lc}
\toprule
Method & Dead Feature Rate $\downarrow$ \\
\midrule
TopK       & 0.11\% \\
Batch TopK & \textbf{0.06\%} \\
Matryoshka & \underline{0.10\%} \\
Ort        & 4.61\% \\
C$^2$R     & 0.90\% \\
\bottomrule
\end{tabular}
\caption{Dead feature rates on Gemma-2-2B layer 12. C$^2$R maintains over 99\% dictionary utilization.}
\label{tab:dead_features}
\end{table}

\end{document}